\documentclass{article}
\usepackage{lmodern}
\usepackage{graphicx} 

\usepackage[a4paper, total={5.5in, 8in}]{geometry}

\usepackage{tikz}
\usetikzlibrary{positioning}

\usepackage{amsfonts}
\usepackage{amsthm}
\usepackage{amsmath}
\usepackage{changepage}
\usepackage{amssymb}
\usepackage{float}
\usepackage{authblk}

\usepackage{comment}

\usepackage{xcolor}

\usepackage{tikz-cd}

\usepackage[colorlinks=true, linkcolor=blue, citecolor=blue, urlcolor=blue]{hyperref}

\newtheorem{theorem}{Theorem}
\newtheorem{example}{Example}
\newtheorem{lemma}{Lemma}
\newtheorem{observation}{Observation}
\newtheorem{proposition}{Proposition}
\newtheorem{definition}{Definition}
\newtheorem{corollary}{Corollary}

\title{Approximate Homomorphisms and Convergent Representations in Transducers}
\author[1,2]{Santiago Cifuentes}
\date{August 2026}

\affil[1]{Dovetail Research Group}
\affil[2]{ICC CONICET, Universidad de Buenos Aires}

\newcommand{\N}{\mathbb{N}}

\newcommand{\states}{\mathcal{S}}
\newcommand{\inputalph}{\mathcal{A}}
\newcommand{\outputalph}{\mathcal{O}}

\begin{document}

\maketitle

\begin{abstract}
    We study the stability of minimal representations of controlled stochastic processes (in particular, \textit{transducers}) under perturbations. This question is motivated by recent experiments finding predictive-state structure in the latent representations of neural networks. We consider standard, linear and predictive transducers. We introduce notions of approximate homomorphism capturing local structural similarity between them, together with metrics comparing their induced dynamics (which we refer to as \textit{interfaces}), and prove properties such as composability of the approximate homomorphisms. For standard transducers, we show that there exist simple interfaces for which there is no approximate homomorphism between the different implementations of the dynamics. In contrast, for every finite-rank interface $\mathcal I$, we prove that all minimal linear transducers implementing interfaces sufficiently close to $\mathcal I$ have an approximate homomorphism to the minimal implementation of $\mathcal I$, with error linear in the perturbation size. We prove an analogous stability result for predictive transducers under a residual metric using some mild hypothesis regarding the indistinguishability of the belief states. These results identify conditions under which canonical transducer representations are robust to perturbations, while showing that such convergence fails without additional structural restrictions. Under the assumption that these type of abstractions are embedded into the hidden layers of modern AI models, this gives some theoretical support to the hypothesis that their latent representations exhibit structural convergence.
\end{abstract}

\section{Introduction}

Since the beginning of neural networks, many AI architectures have included some type of hidden or intermediate layers between the input and output gates in which the models can encode partial results of their computation. In these layers the models usually learn, through training, to represent latent variables and general information useful for their goals  \cite{bengio2013representation,bereska2024mechanistic,lecun2015deep}. 

Although the training process is not deterministic and depends on things such as the training algorithm, the initial parameters and the choice of hyperparameters, it has been observed from the beginning of the deep learning revolution that some structure of these hidden layers coincides between different models \cite{li2015convergent, morcos2018insights}, even when they are implemented in different architectures. There are different ways of measuring this similarity, but for most of them this ``convergent phenomenon'' can be found \cite{klabunde2025similarity}. To name a few of these metrics, similarity can be measured by comparing the distribution of latent vectors inside each layer \cite{chen2026transferring, kornblith2019similarity}, by comparing functional aspects of the different layers \cite{klabunde2025similarity}, or by finding a linear transformation able to transfer features from one model to another \cite{bansal2021revisiting,chen2026transferring,csiszarik2021similarity,gorbett2026characterizing}. These experimental results have motivated the recent proposal of the ``Platonic Representation Hypothesis'' \cite{huh2024platonic}: the idea that ``Neural networks, trained with different objectives
on different data and modalities, are converging to a
shared statistical model of reality in their representation spaces''. Although it is unclear to what extent this hypothesis may hold \cite{ciernik2024objective,ding2021grounding,groger2026revisiting}, most results suggest that some kind of convergence can sometimes be found in different models trained for a similar task.

Overall, there are three core hypotheses which can help to understand this situation \cite{huh2024platonic}. First, the \textit{simplicity bias} hypothesis states that deep AI architectures trained through stochastic gradient descent have a tendency to converge to structurally simple representations which usually generalize well \cite{berchenko2024simplicity, kalimeris2019sgd, valle2018deep}. Meanwhile, the \textit{capacity hypothesis} states that as models include more parameters they encompass a larger set of possible behaviours, and thus it is more likely for different architectures to have a non-empty intersection regarding the instantiations they allow \cite{huh2024platonic}. Finally, the \textit{multitask hypothesis} states that there are fewer representations capable of performing multiple tasks at the same time \cite{cao2021explanatory,nguyen2020wide}, and thus as we train models for more complex goals the optimal configurations become sparser. 

There exists an ongoing theoretical program trying to give support to these hypotheses. For instance, ideas such as implicit regularization \cite{hu2024understanding} or neural collapse \cite{jiang2023generalized}, or entire frameworks such as singular learning theory \cite{watanabe2018mathematical} try to explain how modern gradient descent finds robust representations in current deep learning architectures although the number  of model parameters allows for overfitting. A different approach relies on the idea of a \textit{world model} 
\cite{ha2018world,hafner2023mastering,
rosas2025ai}. More precisely, different empirical and theoretical results support the idea that modern agents develop an internal mechanism equivalent to a description of the dynamics of the environment surrounding them \cite{cifuentes2026general,nayebi2026capable,richens2025generalagents,shai2026transformers}. If we formalize these structures using some mathematical abstraction, then we can ask the question of whether the set of these abstractions implementing the same dynamics has some shared structure \cite{rosas2025ai}. If the answer is positive, this provides some support for the idea that agents learning from a similar data distribution should share some aspects of their internal representations. See Figure~\ref{fig:diagram_world_model_approach} for a diagrammatic sketch of this idea.

\begin{figure}
    \centering
    \includegraphics[width=1\linewidth]{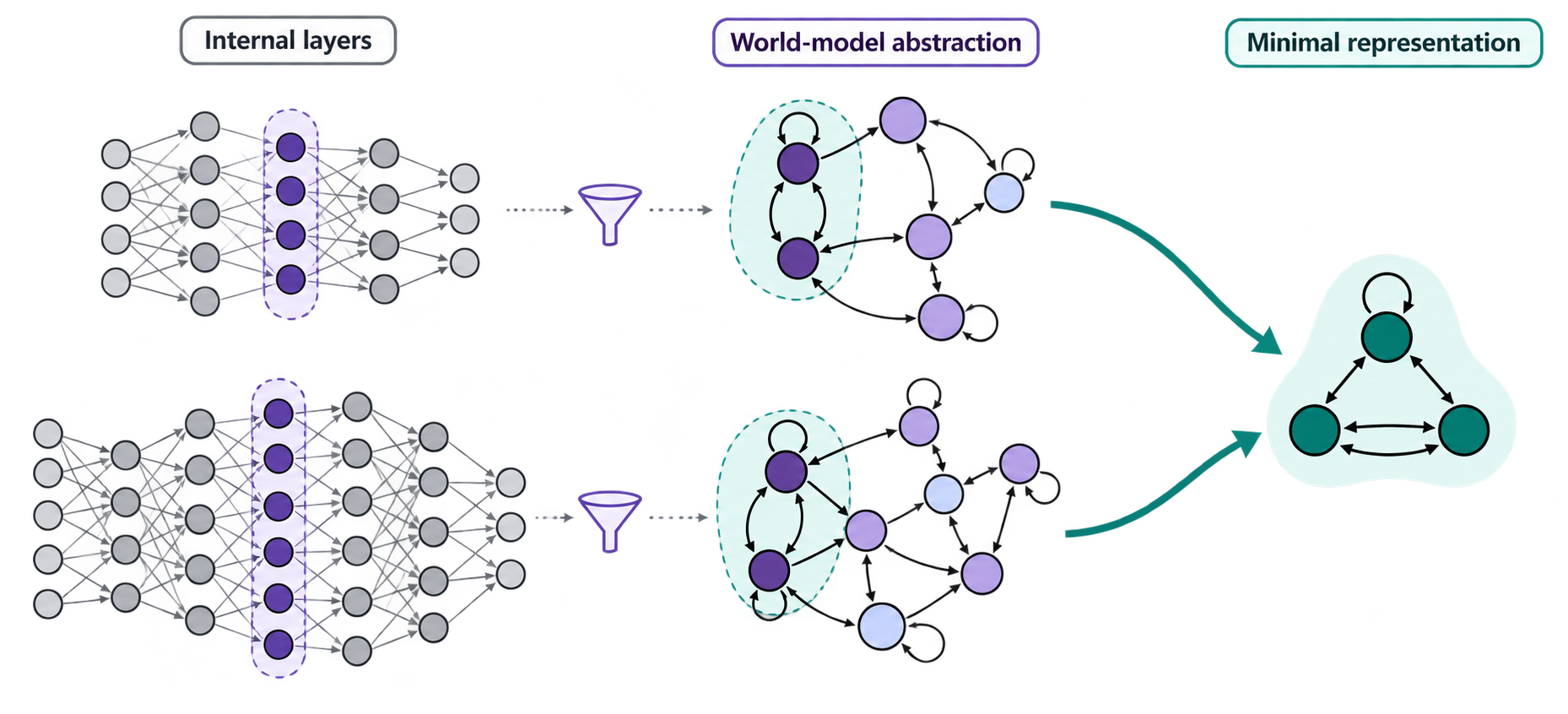}
    \caption{The world model-based approach for understanding convergent structure. We expect that each internal layer of a neural network architecture encodes a world model through some mathematical abstraction (such as a transducer). Then, we search for convergent structure in these abstractions by looking for a minimal model representing the dynamics.}
    \label{fig:diagram_world_model_approach}
\end{figure}

In this work we investigate this idea using \emph{transducers} as an abstraction of world models. A
transducer is a controlled stochastic system with hidden states, inputs, and
outputs, which induces an \emph{interface}: for every finite sequence of
interventions, the interface specifies a probability distribution over the
corresponding sequence of observations in the real world. These types of structures have been studied recently as tools to formalize world models
\cite{boyd2025monoliths,rosas2025ai}, and different experimental results support the idea that modern AI agents implement this type of structure in their residual stream \cite{shai2026transformers,shai2024transformers}. Moreover, in \cite{rosas2025ai} it was proven that in some situations the set of transducers implementing a specified behaviour has a unique minimal implementation such that all other implementations can be structurally mapped into the minimal one. From our perspective, this is a positive result encouraging the possibility of convergent structure.

Our goal is to improve this type of result by weakening some of its hypotheses. More precisely, we aim to improve the result by making it robust to noise and approximation. Consider that we have two transducers implementing a \textit{similar} behaviour (measured through some proper metric). Then, is it the case that they share some structure, measured through some other metric? Note that in practical scenarios we expect different models to learn from slightly different datasets, and thus we need the convergence of transducers to hold also in the case in which they implement slightly different interfaces.

\paragraph{Our Contributions.}
We consider three types of transducers: the ``standard'' ones, linear transducers and predictive transducers, and do the following:

\begin{enumerate}
    \item We give a notion of approximate homomorphism for all types of transducers that allows to decide when a transducer is $\varepsilon$-similar to another one. Our notion is robust with respect to composition and preserves the dynamics that the transducers represent up to an error that scales as $O(\varepsilon)$ under discounted metrics (i.e. metrics that weight differences in long-horizon predictions using a discount factor that decays exponentially with the number of steps).

    \item In the context of standard transducers, we show that there exist interfaces such that the standard transducer implementing these interfaces are structurally far away (using our notion of approximate homomorphism to measure distance). This result shows that the result from \cite{rosas2025ai} regarding the non-existence of a minimal representation for normal transducers cannot be salvaged by introducing an error term in the homomorphism.

    \item For linear transducers, we show that for a significant subset of interfaces it is the case that all minimal linear transducers implementing an interface $\varepsilon$-similar (for a sufficiently small $\varepsilon>0$) to interface $\mathcal{I}$ can be mapped to a common linear transducer introducing some error that scales as $\Gamma_\mathcal{I}\varepsilon$, where $\Gamma_\mathcal{I}$ represents a constant depending only on $\mathcal{I}$. This shows that the minimality of linear implementations is robust to noise in a small neighbourhood of the interface. This result is obtained by inspecting the canonical construction of the minimal linear transducer, which is obtained by working with the \textit{Hankel matrix} of the interface.

    \item Finally, we obtain an analogous result for predictive transducers by leveraging the construction of the minimal predictive implementation based on the notion of a $\epsilon$-machine from computational mechanics \cite{barnett2015computational}.
\end{enumerate}

Taken together, these results show that the existence of minimal representations (and thus of potential convergent structure) is robust to error for the family of linear and predictive transducers. Thus, we provide more theoretical support for the Platonic Representation Hypothesis under the hypothesis that world models show up inside the internal structure of modern AI architectures in the form of these types of abstractions. In the case of predictive transducers, this was partially observed empirically \cite{shai2026transformers,shai2024transformers}. Meanwhile, although there is no previous experiment finding linear transducers within the hidden layers of modern models, we suspect that these architectures should leverage the fact that any representation they contain is embedded in a linear space, and thus it is natural for them to prefer linear transducers over predictive ones. We believe that a fruitful direction for future work would be to reproduce the set-up of~\cite{shai2024transformers} but looking for a minimal linear transducer (or a functionally equivalent mechanism) inside the residual stream.

The remainder of the paper is organized as follows. Section~\ref{sec:old_definitions} reviews
standard, linear, and predictive transducers together with their exact
notions of reduction and minimality. Section~\ref{sec:new_definitions} introduces approximate
homomorphisms, metrics on interfaces, and the composability and continuity
results. Section~\ref{sec:approx_reduction_orders} studies approximate common minima for nearby interfaces,
giving respectively the negative result for unrestricted transducers and the
positive results for linear and predictive transducers. Finally, in Section~\ref{sec:conclusion} we conclude the paper by summarizing our results, discussing their limitations and describing future lines of research. All proofs are deferred to the Appendix to improve readability.

\paragraph{Related work.} Finite-state transducers have a long history in automata theory, beginning
with the Mealy and Moore machines
\cite{mealy1955method,moore1956gedanken}. In their deterministic form, they
describe systems whose internal state is updated in response to an input
while producing an output. Weighted and probabilistic variants replace
deterministic transitions by numerical weights or stochastic kernels, and
have been extensively studied in formal language theory
\cite{mohri1997finite}. The stochastic transducers considered here are closely related to controlled
Markov models, and in particular they resemble Markov Decision Processes (MDPs)
\cite{puterman2014markov}.

To identify convergent structure we use the notion of homomorphism, which corresponds to a  map from one transducer to another that preserves local structure. This type of ``coarse-graining'' operations have a long history in the different abstractions we mentioned before. For Markov chains,
classical lumpability identifies states whose transition probabilities agree
after aggregation \cite{kemeny1969finite}. Probabilistic bisimulation gives
a related behavioural equivalence for labelled probabilistic transition
systems \cite{larsen1991bisimulation}. In the MDP literature, \cite{givan2003equivalence} develops exact state equivalences and
model minimization based on bisimulation, while~\cite{ravindran2003smdp} formulates MDP and semi-MDP homomorphisms as
maps that preserve rewards and aggregate transition probabilities. These
notions have subsequently been organized into broader taxonomies of state
abstraction \cite{li2006towards}.

We will study notions of homomorphism that allow for some error, and thus we refer to them as approximate homomorphisms. In the context of MDPs, such approximate reductions have already been considered 
\cite{abel2016near,ravindran2004approximate,taylor2008bounding}, and our definitions as well as our robustness results (regarding composition and preservation of the interface up to discounted metrics) have analogues in the literature. In probabilistic transition
systems, the notion of approximate bisimulation has a long history~\cite{desharnais2004metrics,ferns2004metrics} and remains an active area of research~\cite{kiefer2021approximate,spork2024spectrum}.

The notion of the Hankel matrix of a process was introduced in the context of weighted automata to construct minimal linear implementations
\cite{schutzenberger1961definition}. In that setting, the minimal realizations obtained are unique up to an
invertible linear change of coordinates (i.e. a base change). Regarding predictive transducers, we use tools from computational mechanics
\cite{crutchfield1989inferring,shalizi2001computational} to obtain minimal representations. In particular, the extension of computational mechanics to input-output processes~\cite{barnett2015computational} can be applied almost directly to our context.

\section{Types of transducers}\label{sec:old_definitions}

In this section we describe the different types of transducers that we will consider in this paper.

\subsection{``Standard'' Transducers}

We use \textit{transducers} to model world models.

\begin{definition}
    A \textbf{transducer} is given by a tuple $(\states, \inputalph, \outputalph, \kappa, p)$ where $\states$ is a set of states, $\inputalph$ is a set of actions (or inputs), $\outputalph$ is a set of reactions (or outputs), $\kappa$ is a Markov kernel\footnote{In this context, a Markov kernel is simply a set of conditional distributions.} of the form $\{\kappa_\tau(s', o|s, a) : s,s'\in \states, a \in \inputalph, o \in \outputalph, \tau \in \N\}$ and $p \in \Delta(\states)$ is an initial distribution over the set of states. 
\end{definition}

Transducers represent a world model indicating, for every possible sequence of actions $a_1 \ldots a_n$, a distribution on the reaction $o_1 \ldots o_n$ of the environment. They use (hidden) states to keep track of the previous events, and the function $\kappa$ describes the relation between actions and outputs and how the state is updated. This function can depend on the timestep $\tau \in \N$, but for simplicity we will assume that $\kappa_\tau = \kappa_0$ for every $\tau$ (this corresponds to assuming \textit{stationary} dynamics). Also, we may omit $\kappa$ in the notation and simply talk about the probabilities of the events. For example, we write $\Pr{}_{T}\left(o| s, a\right)$ to denote the value $\sum_{s' \in \states} \kappa(s', o | s, a)$, or similarly $\Pr{}_T\left(s' | s,a,o\right)$ to denote $\kappa(s', o | s,a ) / \Pr{}_T(o | s,a)$ whenever the denominator is positive. We will assume for simplicity that $\states, \inputalph$ and $\outputalph$ are countable.

\newcommand{\interface}{\mathcal{I}}

A \textit{trace} over $T$ is a finite or infinite sequence of outputs.
As mentioned, any transducer defines a probability for each trace \textit{conditioned} on each sequence of actions. We refer to such a description (i.e. a list of probabilities $\Pr\left( o_1 \ldots o_n | \textbf{a}\right)$ for every finite sequence $o_1\ldots o_n \in \outputalph^n$ and infinite sequence $\mathbf{a} \in \inputalph^\omega$) as an \textbf{interface} $\interface$. We will only be interested in anticipation-free interfaces, i.e. those that satisfy $\Pr\left(o_1\ldots o_n | \textbf{a}\right) = \Pr\left(o_1\ldots o_n | a_1\ldots a_n\right)$. Anticipation-free interfaces coincide exactly with interfaces ``implementable'' by transducers \cite{rosas2025ai}[Lemma 4].

More precisely, given $o_1\ldots o_n$ and $a_1 \ldots a_n$ the (conditioned) probability that the transducer $T$ induces can be computed as
\begin{align}\label{eq:induced_interface_by_transducer}
    \Pr{}_T\left(o_1\ldots o_n | a_1 \ldots a_n\right) = \sum_{s_0 \ldots s_n \in \states^{n+1}} p(s_0) \prod_{t=1}^n \kappa(s_t, o_t | s_{t-1}, a_t)
\end{align}

Whenever $|\states|,|\inputalph|,|\outputalph| < \infty$ this computation can be simplified: if $M_{a,o} \in \mathbb{R}^{|\states| \times |\states|}$ is given by $M_{a, o}(s,s') = \kappa(s', o | s, a)$, then
\begin{align*}
    \Pr{}_T (o_1\ldots o_n | a_1\ldots a_n) = p M_{a_1, o_1} \ldots M_{a_n, o_n} \mathbf{1}
\end{align*}

\begin{example}\label{ex:transducer}
    \normalfont Figure~\ref{fig:example_transducer} shows a transducer with deterministic dynamics (i.e. for every $s\in \states, a\in \inputalph$ there is some $s'\in \states$ and $o \in \outputalph$ such that $\kappa(s', o| s, a) = 1$). The states are $\states = \{s_0, s_1\}$, the actions $\inputalph = \{\texttt{continue}, \texttt{stay}\}$, and the possible outputs $\outputalph = \{0, 1\}$. The initial distribution is concentrated in state $s_0$. It represents a system that outputs $010101\ldots$ indefinitely as long as the action $\texttt{continue}$ is chosen at each step. If $\texttt{stay}$ is employed instead, the dynamics are ``frozen'' for one step.

    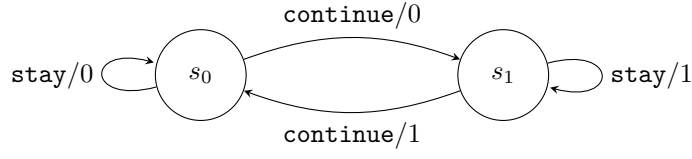
\begin{figure}
        \centering
        \begin{tikzpicture}[
            state/.style={circle, draw, minimum size=1.2cm, align=center},
            >=stealth,
            node distance=4cm
        ]
            
            \node[state] (q0) {$s_0$};
            \node[state, right of=q0] (q1) {$s_1$};
            
            \path[->]
                (q0) edge[loop left] node[left] {\texttt{stay}/0} (q0)
                (q1) edge[loop right] node[right] {\texttt{stay}/1} (q1)
                (q0) edge[bend left=20] node[above] {\texttt{continue}/0} (q1)
                (q1) edge[bend left=20] node[below] {\texttt{continue}/1} (q0);
        
        \end{tikzpicture}
        \caption{Example of a ``deterministic'' transducer. Each edge contains an action $a$ followed by an output $o$. An edge from $s$ to $s'$ with label $a/o$ indicates that $\kappa(s', o | s, a) = 1$.}
        \label{fig:example_transducer}
    \end{figure}
\end{example}

We will assume that all the states of a transducer are reachable from some state with initial positive probability. Namely, for every state $s$ there must exist a state $s_0$ such that $p(s_0) > 0$ and a sequence of actions $a_1 \ldots a_k$ such that $\Pr_{T}\left( s| s_0, a_1\ldots a_k\right) > 0$. Also, we will sometimes use $\Sigma = \inputalph \times \outputalph$.

There is a well-defined notion of \textbf{homomorphism} for these objects, which allows us to coarse-grain states as well as input and output symbols.

\begin{definition}\label{def:transducer_homomorphism}
    Given two transducers $T_1 = (\states_1, \inputalph_1, \outputalph_1, \kappa_1, p_1)$ and $T_2 = (\states_2, \inputalph_2, \outputalph_2, \kappa_2, p_2)$, a \textbf{homomorphism} is given by three mappings $\langle \phi:\states_1 \to \states_2, f:\inputalph_1 \to \inputalph_2, g:\outputalph_1 \to \outputalph_2\rangle$ satisfying
    \begin{align}\label{eq:exact_kernel_pushforward}
        \kappa_2(s_2,o_2|\phi(s_1),f(a_1))
        &=
        \sum_{\substack{s'\in \phi^{-1}(s_2)\\ o'\in g^{-1}(o_2)}}
        \kappa_1(s',o'|s_1,a_1),\\
        p_2(s_2) &= \sum_{s_1\in \phi^{-1}(s_2)}p_1(s_1),\label{eq:exact_initial_pushforward}
    \end{align}
    for every $s_1\in \states_1$, $a_1\in \inputalph_1$, $s_2\in \states_2$, and $o_2\in \outputalph_2$.
\end{definition}

\newcommand{\idfun}{\texttt{id}}

Condition~\eqref{eq:exact_kernel_pushforward} says that the joint one-step distribution on the next state and output is preserved after applying the coarse-grainings $\phi$ and $g$ (and translating actions through $f$). Condition~\eqref{eq:exact_initial_pushforward} makes sure that the initial distributions are equivalent up to $\phi$.

If we require $\outputalph_1 = \outputalph_2$ and $\inputalph_1 = \inputalph_2$ then both transducers have the same ``type''. Moreover, if also $f = g = \idfun{}$ and $\phi$ is surjective we say that the homomorphism is a \textit{reduction}.

\begin{example}\label{ex:bigger_transducer}
    \normalfont Consider the transducer from Figure~\ref{fig:example_bigger_transducer}. There is a reduction from this transducer to the one from Figure~\ref{fig:example_transducer}: define $\phi$ as $\phi(s_0) = \phi(s_2) = s_0$ and $\phi(s_1) = \phi(s_3) = s_1$, while taking $f =g =\idfun{}$. In some sense, the transducer from Figure~\ref{fig:example_bigger_transducer} implements the interface in an ``inefficient'' manner.

    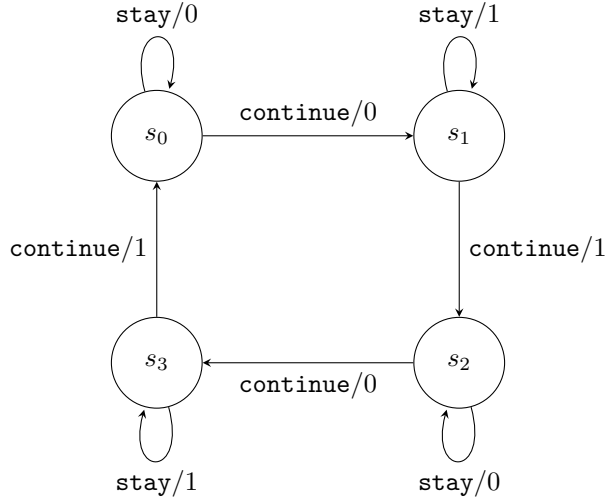
\begin{figure}
        \centering
            \begin{tikzpicture}[
            state/.style={circle, draw, minimum size=1.2cm, align=center},
            >=stealth,
            node distance=4cm
        ]
            
                \node[state] (x0) at (0,1.5) {$s_0$};
                \node[state] (x1) at (4,1.5) {$s_1$};
                \node[state] (y1) at (4,-1.5) {$s_2$};
                \node[state] (y0) at (0,-1.5) {$s_3$};
                
                \path[->] (x0) edge[loop above] node {$\texttt{stay}/0$} (x0);
                \path[->] (x1) edge[loop above] node {$\texttt{stay}/1$} (x1);
                \path[->] (y1) edge[loop below] node {$\texttt{stay}/0$} (y1);
                \path[->] (y0) edge[loop below] node {$\texttt{stay}/1$} (y0);
                
                \path[->] (x0) edge node[above] {$\texttt{continue}/0$} (x1);
                \path[->] (x1) edge node[right] {$\texttt{continue}/1$} (y1);
                \path[->] (y1) edge node[below] {$\texttt{continue}/0$} (y0);
                \path[->] (y0) edge node[left] {$\texttt{continue}/1$} (x0);
                
            \end{tikzpicture}
        \caption{Example of a ``deterministic'' transducer that implements the same interface as the one from Figure~\ref{fig:example_transducer}. An edge from $s$ to $s'$ with label $a / o$ indicates that $\kappa(s', o|s, a) = 1$.}
        \label{fig:example_bigger_transducer}
    \end{figure}
\end{example}

Reductions can be composed. Thus, after fixing an interface $\interface$ we can look at the set of transducers implementing $\interface$, and if we quotient them properly (identifying transducers $T_1$ and $T_2$ such that there are reductions both from $T_1$ to $T_2$ and from $T_2$ to $T_1$) then the reduction relation gives the set a \textit{poset} structure.

In \cite{rosas2025ai} some properties of these posets are proven, and in particular the fact that in general they need not have a unique minimum. This situation can be salvaged in at least two ways. First, if we consider linear transducers (which allow for ``negative'' probabilities), then uniqueness of the minimum can be proven \cite{rosas2025ai}[Theorem 2]. Second, we can restrict attention to the subposet of \textit{predictive} transducers (intuitively, those whose state transitions are deterministic given the last state, action and output): in that case, there is a unique minimum, and it coincides with the $\epsilon$-machine from computational mechanics \cite{barnett2015computational} representing the dynamics \cite{rosas2025ai}[Theorem 3].

Before proceeding, we note that the definition of homomorphism we introduced is not exactly the same as the one from \cite{rosas2025ai}. In Appendix~\ref{apx:comparison_homo} we compare them and show nonetheless that they coincide when we restrict to reductions. Later we will see that our proposal is easier to extend to the approximate setting. In particular, Condition~\eqref{eq:exact_kernel_pushforward} states that $\kappa_2$ must be equal to the pushforward of $\kappa_1$ through $\phi$ and $g$. Thus, we can introduce an error term in the homomorphism by comparing $\kappa_2$ to this pushforward using any distance between distributions.

\subsection{Linear transducers}

We will also consider \textit{linear transducers}: a model of a transducer in which the states are embedded in a vector space. From now on we define, for every interface $\interface$, its associated formal series: for every $w \in \Sigma^* = (\inputalph \times \outputalph)^*$ let
\[
    F_\interface(w) = F_{\interface}((a_1, o_1)\ldots (a_n, o_n))
    =
    \Pr{}_{\interface}(o_1\cdots o_n\mid a_1\cdots a_n),
    \qquad F_\interface(\epsilon)=1,
\]
where $\epsilon$ denotes the empty word. 

\begin{definition}
    A \textbf{linear transducer} over $(\inputalph,\outputalph)$ is a tuple $G=(V,\xi,\lambda,\{M_\sigma\}_{\sigma\in\Sigma})$,
    where $V$ is a real vector space, $\xi\in V$ is an initial vector, $\lambda\in V^*$ is a
    linear functional, and each $M_\sigma:V\to V$ is a linear map. For a word
    $w=\sigma_1\cdots\sigma_n$ we define $M_w$ recursively by
    \[
        M_\epsilon=\mathrm{id}_V,
        \qquad
        M_{w\sigma}=M_\sigma M_w.
    \]
    The linear transducer generates the formal series
    \[
        F_G(w)=\lambda(M_w\xi).
    \]
    We say that $G$ implements an interface $\interface$ if $F_G=F_\interface$. We will assume without loss of generality that $V=\operatorname{span}\{M_w\xi:w\in\Sigma^*\}$, i.e. that the whole space $V$ is ``used'' by the transducer\footnote{We add this hypothesis to improve the clarity of our exposition. All results still hold when removing this condition.}.
\end{definition}

The size of a linear transducer is measured by the dimension of the linear space needed
to represent the series, which can be infinite. 

We now introduce the concept of the \textit{Hankel matrix} of the interface.

\begin{definition}
    The \textbf{Hankel matrix} of $\interface$ is the infinite matrix
    \[
        H_\interface(u,v)=F_\interface(uv),
        \qquad u,v\in\Sigma^* .
    \]
    For each prefix $u\in\Sigma^*$ define the row $h_\interface(u):\Sigma^*\to\mathbb R$ such that $h_\interface(u)(v)=F_\interface(uv)$. Then, the dimension of $\interface$ is
    \[
        \dim(\interface)=\operatorname{rank}(H_\interface)
        =\dim \operatorname{span}\{h_\interface(u):u\in\Sigma^*\}.
    \]
    If this rank is finite, we call $\interface$ a finite-rank interface.
\end{definition}

A linear transducer implementing an interface can be obtained from its Hankel matrix. Let
\[
    V_\interface=\operatorname{span}\{h_\interface(u):u\in\Sigma^*\}.
\]
Pick $\xi_\interface=h_\interface(\epsilon)$ and let
$\lambda_\interface:V_\interface\to\mathbb R$ be evaluation at the empty suffix as
\[
    \lambda_\interface(r)=r(\epsilon).
\]
For every $\sigma\in\Sigma$, define the shift operator
$R^\interface_\sigma:V_\interface\to V_\interface$ by
\[
    R^\interface_\sigma h_\interface(u)=h_\interface(u\sigma),
\]
and extend linearly. This is well-defined: if $\sum_i \alpha_i h_\interface(u_i)=0$, then
for every suffix $v$,
\[
    \sum_i \alpha_i h_\interface(u_i\sigma)(v)
    =
    \sum_i \alpha_i F_\interface(u_i\sigma v)
    =
    \sum_i \alpha_i h_\interface(u_i)(\sigma v)
    =0.
\]
Thus, it follows that $G_\interface=(V_\interface,\xi_\interface,\lambda_\interface,\{R^\interface_\sigma\}_{\sigma\in\Sigma})$ is a linear transducer, and $\lambda_\interface(R^\interface_w\xi_\interface)=F_\interface(w)$ for every word $w$.

This representation is minimal: if
$G=(V,\xi,\lambda,\{M_\sigma\})$ implements $\interface$, then
\[
    h_\interface(u)(v)=F_\interface(uv)=\lambda(M_v M_u\xi).
\]
for every word $u, v$. Hence all Hankel rows are obtained from vectors $M_u\xi\in V$, so
\[
    \operatorname{rank}(H_\interface)\leq \dim V.
\]
Note that minimal linear transducers are unique up to invertible
linear changes of coordinates.

In the context of linear transducers we will use \textit{linear reductions} to formalize the idea of homomorphisms between models.

\begin{definition}\label{def:linear_reduction}
Let $G=(V,\xi,\lambda,\{M_\sigma\}_{\sigma\in\Sigma})$ and $G'=(W,\xi',\lambda',\{N_\sigma\}_{\sigma\in\Sigma})$
be linear transducers over the same input and output alphabets. A \textbf{linear
reduction} from $G$ to $G'$ is a surjective linear map $L:V\to W$ satisfying
\begin{enumerate}
    \item $L\xi=\xi'$.
    \item $L M_\sigma=N_\sigma L \text{ for every }\sigma\in\Sigma$.
    \item $\lambda' L=\lambda$.
\end{enumerate}
\end{definition}

The first condition preserves the initial vector, the second says that $L$ translates the
internal dynamics, and the third preserves the ``reading'' of the vectors. These
conditions imply that $\lambda'(N_w\xi')=\lambda(M_w\xi)$ for every word $w$.

By the construction above, it can be proven that any linear transducer $G$ implementing $\interface$ can  be reduced to $G_{\interface}$.

\providecommand{\interfaceJ}{\mathcal{J}}

\begin{lemma}\label{lem:prediction-map}
    Let $G=(V,\xi,\lambda,\{M_\sigma\})$ be a linear transducer implementing
    an interface $\interface$. Then there is a linear reduction $\rho_G:G\to G_\interface$
    given by
    \[
        \rho_G(M_w\xi)=h_\interface(w).
    \]
\end{lemma}

\begin{example}\label{ex:linear_alternating}
\normalfont The deterministic transducer $T$ from Figure~\ref{fig:example_transducer}
admits a simple linear representation. Let
\[
V=\mathbb{R}^2,
\qquad
\xi=
\begin{pmatrix}
1\\0
\end{pmatrix},
\qquad
\lambda(x_0,x_1)=x_0+x_1,
\]
where the two standard basis vectors represent the states $s_0$ and
$s_1$. Consider the transition maps
\[
M_{\mathsf{continue},0}
=
\begin{pmatrix}
0&0\\
1&0
\end{pmatrix},
\qquad
M_{\mathsf{continue},1}
=
\begin{pmatrix}
0&1\\
0&0
\end{pmatrix},
\]
and
\[
M_{\mathsf{stay},0}
=
\begin{pmatrix}
1&0\\
0&0
\end{pmatrix},
\qquad
M_{\mathsf{stay},1}
=
\begin{pmatrix}
0&0\\
0&1
\end{pmatrix}.
\]
Then, the linear transducer $G = (V, \xi, \lambda, (M_\sigma)_{\sigma \in \Sigma})$ implements the same interface as $T$.
\end{example}

In general, every standard transducer $T$ can be transformed into a linear transducer whose underlying space has dimension equal to the number of states of $T$.

\subsection{Predictive transducers}
\label{subsec:predictive_transducers}

We finally consider predictive transducers, which correspond to transducers whose internal states do not contain predictive information that is unavailable from the observable input--output history. From now on, we say that a history $h = (a_1,o_1) \ldots (a_n, o_n) \in \Sigma^*$ is \emph{admissible} for an interface $\interface$ if
$\Pr_{\interface}(o_1\cdots o_n\mid a_1\cdots a_n)>0$. The empty history is always admissible. For any admissible history $h=(a_1,o_1)\ldots(a_n,o_n)$ of an interface $\interface$, we consider the \textit{residual interface} $\interface^h$ as the interface obtained by conditioning on $h$: for every $u\in\inputalph^m$ and $v\in\outputalph^m$,
\begin{align}
    \Pr{}_{\interface^h}(v\mid u)
    :=
    \frac{
        \Pr{}_{\interface}(o_1\cdots o_n v\mid a_1\cdots a_n u)
    }{
        \Pr{}_{\interface}(o_1\cdots o_n\mid a_1\cdots a_n)
    }.
    \label{eq:residual_interface}
\end{align}

For a transducer $T=(\states,\inputalph,\outputalph,\kappa,p)$ and a state $s\in\states$, let $\interface_{T,s}$ denote the interface generated by the same kernel $\kappa$ with initial distribution $\delta_s$ (i.e. when all probability mass is concentrated on $s$). For an admissible history $h=(a_1,o_1)\ldots (a_n, o_n)$, also write
\[
    q_T(s\mid h)
    :=
    \Pr{}_T(S_n=s\mid o_{1:n},a_{1:n})
\]
for the posterior distribution over the internal state at step $n$ after observing $h$.

\begin{definition}
\label{def:predictive_transducer}
Let $T$ be a transducer. We say that $T$ is \textbf{predictive} if, for every admissible history $h$ and every state $s$ such that $q_T(s\mid h)>0$,
\begin{align}
    \interface_{T,s}=\interface_T^h.
    \label{eq:predictive_transducer}
\end{align}
Equivalently, conditional on the observable history, knowing the current internal state does not change the expected distribution for future events.
\end{definition}

For this class of transducers there is always a minimal implementation of each interface, and it can be constructed explicitly. To do this, identify histories that make exactly the same predictions: for admissible histories $h$ and $h'$, define the \textit{predictive equivalence relation} as
\begin{align*}
    h\sim_{\interface}h'
    \quad\Longleftrightarrow\quad
    \interface^h=\interface^{h'}.
\end{align*}
Denote the equivalence class of $h$ by $[h]_{\interface}$ and let
\[
    \states_{\epsilon}(\interface)
    :=
    \{[h]_{\interface}:h\text{ is admissible for }\interface\}.
\]
The transitions are defined in the expected way in the next definition. This construction corresponds to the notion of an $\epsilon$-machine from computational mechanics \cite{barnett2015computational}.

\begin{definition}
\label{def:epsilon_transducer}
Let $\interface$ be an interface. Its \textbf{$\epsilon$-transducer} is given by
\[
    E(\interface)
    =
    (\states_{\epsilon}(\interface),\inputalph,\outputalph,
      \kappa_{\epsilon},\delta_{[\epsilon]_{\interface}}),
\]
where, for every admissible history $h$, action $a$, and output $o$, we set
\begin{align}
    \kappa_{\epsilon}(s',o\mid[h]_{\interface},a)
    &:=
    \mu_{\epsilon}(o\mid[h]_{\interface},a)
    \mathbf{1}\{s'=\delta_{\epsilon}([h]_{\interface},a,o)\}.
    \label{eq:epsilon_kernel}
\end{align}
where $\mu_{\epsilon}(o\mid[h]_{\interface},a)
    =\Pr_{\interface^h}(o\mid a)$ and $\delta_{\epsilon}([h]_{\interface},a,o)
    =[h(a,o)]_{\interface}$.
\end{definition}

Observe that this transducer evolves \textit{deterministically}: for every state $s$, input $a$ and output $o$ there is a unique next possible state $s'$. This ensures that Eq.~\eqref{eq:predictive_transducer} is satisfied.

The next proposition states that this implementation is the minimal one among the predictive ones.

\begin{proposition} \label{prop:epsilon_transducer_minimal}
    The transducer $E(\interface)$ implements $\interface$ and is predictive. Moreover, if
    $T=(\states,\inputalph,\outputalph,\kappa,p)$ is any predictive transducer implementing $\interface$, then there is a reduction from $T$ to $E(\interface)$.
\end{proposition}

\begin{example}
    \normalfont The transducer from Figure~\ref{fig:example_transducer} is predictive. Moreover, it is also the minimal predictive transducer for that interface.
\end{example}

See Figure~\ref{fig:minimal_models} for a diagram showcasing the structure of the poset of standard, linear and predictive transducers. As already mentioned, due to Lemma~\ref{lem:prediction-map} and Proposition~\ref{prop:epsilon_transducer_minimal} the poset for linear and predictive transducers each has a minimum for every interface. Meanwhile, for the case of standard transducers there are interfaces for which there is no unique minimum.

\begin{figure}[ht]
    \centering
    \includegraphics[width=1\linewidth]{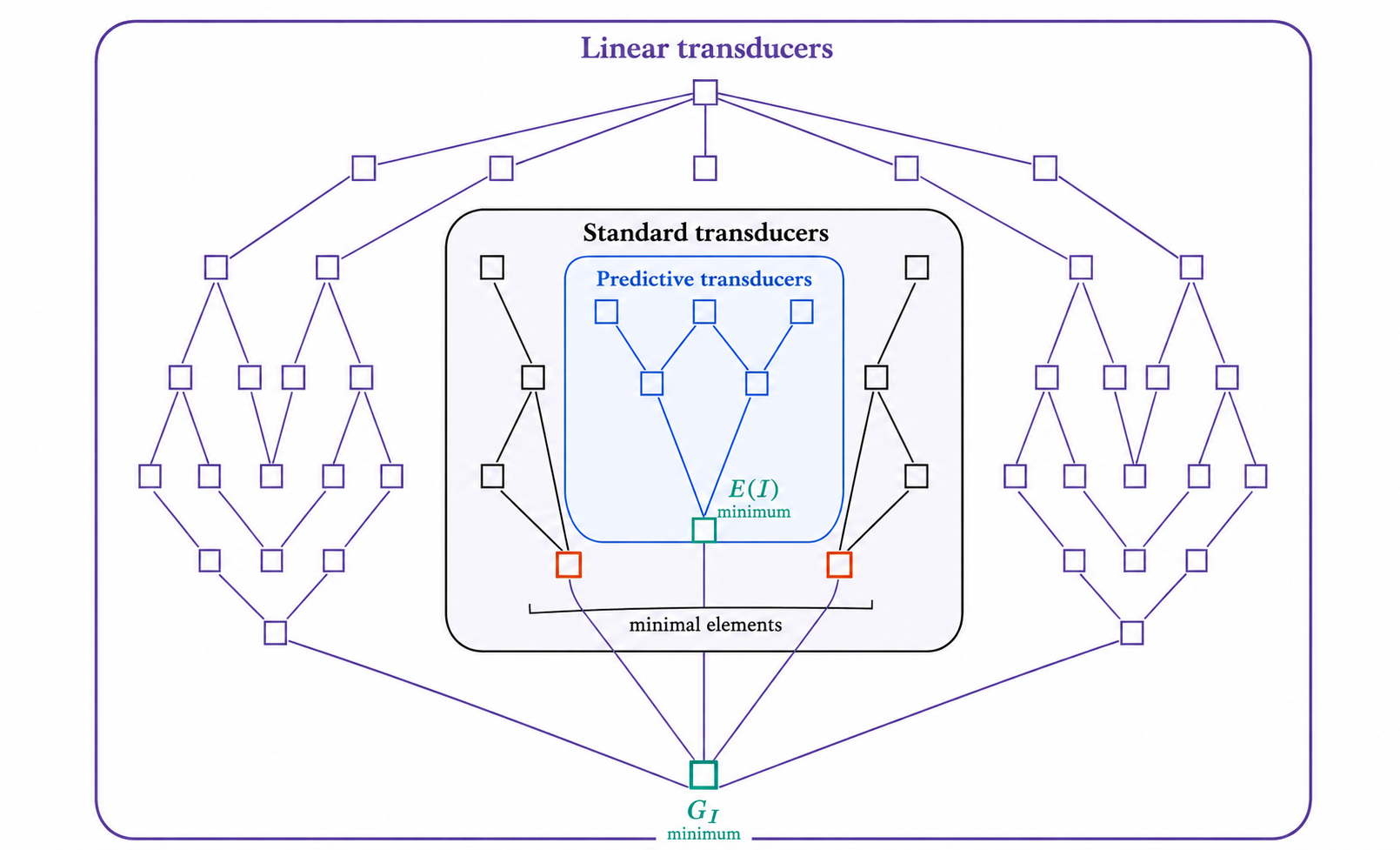}
    \caption{A diagram of the lattice of transducers for standard, linear and predictive implementations. Due to Lemma~\ref{lem:prediction-map} and Proposition~\ref{prop:epsilon_transducer_minimal} the lattice of linear and predictive transducers has a unique minimum, while the one of standard transducers can have more than one minimal element.}
    \label{fig:minimal_models}
\end{figure}

\section{Approximate homomorphisms and the space of interfaces}\label{sec:new_definitions}

In this section we provide approximate variants of the notions of homomorphisms introduced in the previous section, and prove some basic properties.

\subsection{The case of ``standard'' transducers}

The type of coarse-grainings that Definition~\ref{def:transducer_homomorphism} allows is exact in a strong structural sense. It says that the whole one-step mechanism of $T_2$ is obtained by pushing forward the one-step mechanism of $T_1$ along the maps $\phi$, $f$, and $g$. Thus, if two states of $T_1$ are identified by $\phi$, they must have exactly the same coarse-grained output law and exactly the same coarse-grained transition law.

For real world models obtained through learning or other iterative procedures we don't expect them to be structurally identical. Thus, the purpose of approximate homomorphisms is to introduce some degree of error in this notion. We keep the maps $\phi$, $f$, and $g$; but we now allow the push-forward dynamics to differ by some $\varepsilon > 0$.

\begin{definition}\label{def:trans_approx_hom}
    Given two transducers $T_1 = (\states_1, \inputalph_1, \outputalph_1, \kappa_1, p_1)$ and $T_2 = (\states_2, \inputalph_2, \outputalph_2, \kappa_2, p_2)$, a \textbf{$\varepsilon$-homomorphism} is given by three mappings $\langle \phi: \states_1 \to \states_2, f:\inputalph_1 \to \inputalph_2, g: \outputalph_1 \to \outputalph_2\rangle$ satisfying
    \begin{align}
        \left\|(\phi\times g)_*\kappa_1(\cdot,\cdot|s_1,a_1)
        -
        \kappa_2(\cdot,\cdot|\phi(s_1),f(a_1))\right\|_{\mathrm{TV}}
        &\leq \varepsilon
        \label{eq:approx_kernel_pushforward}\\
        \left\|\phi_*p_1-p_2\right\|_{\mathrm{TV}}
        &\leq \varepsilon,\label{eq:approx_initial_pushforward}
    \end{align}
    for every $s_1\in \states_1$ and $a_1\in \inputalph_1$\footnote{Here $(\phi\times g)_*$ and $\phi_*$ denote the push-forwards of the distributions. See Appendix~\ref{subsec:tv_preliminaries} for a precise definition.}.
\end{definition}

The choice of total variation is not completely arbitrary: we will see that due to its properties (which are enumerated in the Appendix~\ref{subsec:tv_preliminaries}) approximate homomorphisms are composable.

\begin{example}\label{ex:approximate_homomorphism}
    \normalfont Consider the actionless transducers from Figure~\ref{fig:example_transucers_and_aprox_homo} with $\varepsilon \in (0,1]$. Each edge has a label $(o, p)$ indicating the probability $p$ of transitioning using that edge and outputting $o$ in the process. There is a $\varepsilon$-reduction from the transducer on the left to the one on the right: take $\phi(s_0) = t_0$ and $\phi(s_1) = \phi(s_2) = t_1$. Meanwhile, there is no 0-reduction (i.e. exact reduction) between them.

    \begin{figure}
    \centering
    \begin{tikzpicture}[>=stealth, thick, every node/.style={font=\large}]
    
      \node[draw, circle, minimum size=1cm] (L1) at (0,0) {$s_0$};
      \node[draw, circle, minimum size=1cm] (T1) at (5,2.6) {$s_1$};
      \node[draw, circle, minimum size=1cm] (B1) at (5,-2.2) {$s_2$};
    
      \draw[->, bend left=20] (L1) to node[above left] {\$, $\frac12$} (T1);
      \draw[->, bend left=20] (T1) to node[below right] {$0,1$} (L1);
    
      \draw[->, bend left=10] (L1) to node[above] {\$, $\frac12$} (B1);
    
      \draw[->, bend left=30] (B1) to node[pos=0.25, above] {$0,1-\varepsilon$} (L1);
      \draw[->, bend left=60] (B1) to node[below] {$1,\varepsilon$} (L1);

      \draw[thick] (6.5,3.5) -- (6.5,-3.2);

      \node[draw, circle, minimum size=1cm] (L2) at (8,0) {$t_0$};
      \node[draw, circle, minimum size=1cm] (M2) at (13,0) {$t_1$};
    
      \draw[->, bend left=22] (L2) to node[above] {\$, $1$} (M2);
      \draw[->, bend left=22] (M2) to node[above] {$0,\,1$} (L2);
    
    \end{tikzpicture}
    \caption{Two actionless transducers (or rather, transducers with a single action $a$), with outputs $\outputalph = \{0, 1, \$\}$. An edge from $s$ to $s'$ with label $o,p$ indicates that $\kappa(s', o | s, a) = p$.}
    \label{fig:example_transucers_and_aprox_homo}
\end{figure}
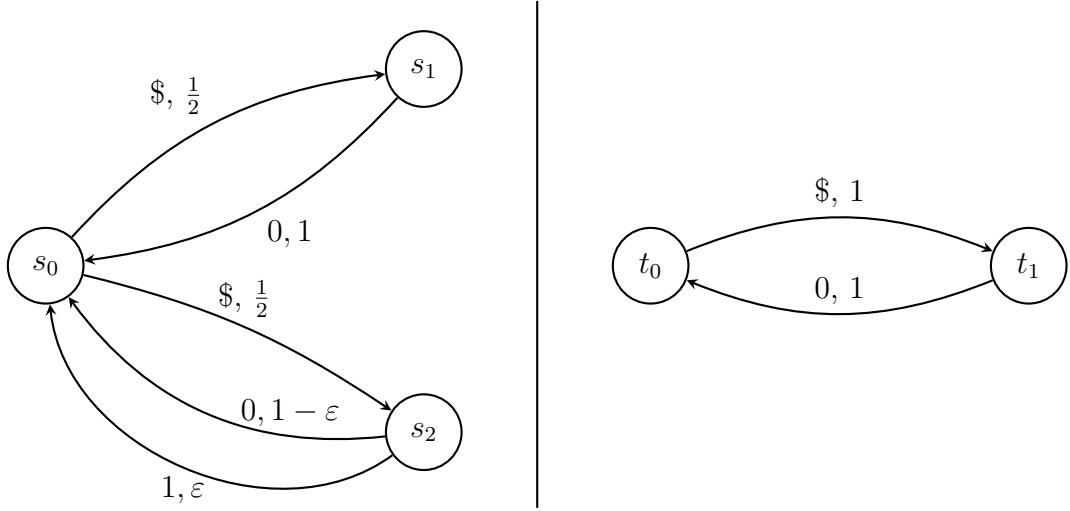
\end{example}

This notion of approximate homomorphism ensures each state $s \in \states_1$ gets mapped to a state whose one-step dynamics are similar after coarse-graining. Thus, if we look at \textit{approximate reductions} (enforcing that $\inputalph_1 = \inputalph_2$, $\outputalph_1 = \outputalph_2$, $f = g = \idfun{}$, and $\phi$ is surjective), one transducer $T_1$ can be approximately reduced to another one $T_2$ only if their states are \textit{locally} similar. Does this imply that the interfaces they induce are also similar? We recall that for exact homomorphisms this is the case.

\begin{observation}
    If there is a reduction from $T_1$ to $T_2$ then $\interface_{T_1} = \interface_{T_2}$ \cite{rosas2025ai}[Lemma 5].
\end{observation}

To approach this question in the approximate setting we need a way to compare different interfaces, i.e. a metric over this space.

Observe that an interface $\interface$ is given essentially by a map $D_\interface:\inputalph^* \to \Delta(\outputalph^*)$ such that $D_\interface(a_1\ldots a_n)$ represents the distribution $\Pr{}_{\interface}\left(\cdot | a_1 \ldots a_n\right)$ which has support over $\outputalph^n$. Then, to define a metric for interfaces we can pick any metric for distributions and then aggregate it over all the possible action sequences in $\inputalph^*$. For instance, we can consider total variation to compare the distributions and aggregate them with the supremum, obtaining
\begin{align}\label{eq:sup_distance_for_interfaces}
    d_\infty(\interface_1, \interface_2) = \sup_{\textbf{a} \in \inputalph^*} ||D_{\interface_1}(\textbf{a}) - D_{\interface_2}(\textbf{a})||_{\mathrm{TV}}
\end{align}

We could also weight each sequence of actions according to its length, reflecting the choice to place less weight on long-horizon discrepancies. Thus, we can consider
\begin{align}\label{eq:discounted_distance_for_interfaces}
    d_\gamma(\interface_1, \interface_2) = \sum_{n=0}^\infty \gamma^n \sup_{\textbf{a} \in \inputalph^n} ||D_{\interface_1}(\textbf{a}) - D_{\interface_2}(\textbf{a})||_{\mathrm{TV}}
\end{align}
for some $\gamma \in (0,1)$. Note that the distances in Eqs.~\eqref{eq:sup_distance_for_interfaces} and~\eqref{eq:discounted_distance_for_interfaces} are indeed well-defined metrics over the set of interfaces.

Since each transducer $T$ induces an interface $\interface_T$ through Eq.~\eqref{eq:induced_interface_by_transducer}, any metric between interfaces can be seen as a pseudometric\footnote{It is a pseudo metric because different transducers implementing the same interface are at distance 0.} between transducers as
\begin{align*}
    d(T_1,T_2) = d(\interface_{T_1}, \interface_{T_2}).
\end{align*}

Are these metrics ``continuous'' with respect to the notion of reduction? Namely, is there a metric $d$ and a function $f: \mathbb{R}_{>0} \to \mathbb{R}_{\geq 0}$ with $f(x) \underset{x \to 0}{\rightarrow} 0$ such that, if there is an $\varepsilon$-reduction from $T_1$ to $T_2$, then $d(\interface_{T_1}, \interface_{T_2}) \leq f(\varepsilon)$? We first observe that this is not the case for the supremum distance in Eq.~\eqref{eq:sup_distance_for_interfaces}.

\begin{example}
    \normalfont Pick $d_\infty$ as in Eq.~\eqref{eq:sup_distance_for_interfaces}, and consider the transducers from Example~\ref{ex:approximate_homomorphism}. Then, if $T_1$ is the transducer on the left and $T_2$ the one on the right, it can be seen that $d(\interface_{T_1}, \interface_{T_2}) = 1$ for every $\varepsilon > 0$.
\end{example}

Intuitively, the supremum distance is not controlled by the approximate homomorphism notion because the error bound applies only to the one-step dynamics. Thus, the interfaces implemented by the two transducers at long horizons (i.e. the distribution $\Pr\left( \cdot | \textbf{a}\right)$ for $\textbf{a} \in \inputalph^n$ with $n \to \infty$) can be arbitrarily far away in metrics such as total variation.

Nonetheless, this observation suggests that the discounted metrics from Eq.~\eqref{eq:discounted_distance_for_interfaces} might be preserved by the approximate homomorphism notion, and indeed this is the case.

\begin{theorem}\label{teo:approximate_red_maintains_discounted}
    Suppose there is an $\varepsilon$-reduction from $T_1$ to $T_2$. Then, if $d_{\gamma}$ is the distance from Eq.~\eqref{eq:discounted_distance_for_interfaces}, it holds that
    \begin{align}\label{eq:discounted_result}
        d_\gamma(\interface_{T_1}, \interface_{T_2}) \leq \frac{\varepsilon}{(1-\gamma)^2}.
    \end{align}
\end{theorem}

We write $T_1 \overset{\varepsilon}{\rightarrow{}} T_2$ to indicate that there is an $\varepsilon$-homomorphism from $T_1$ to $T_2$. As already noted, exact homomorphisms can be composed, and thus the reduction relation is transitive. For approximate homomorphisms we can prove the following additive version.

\newcommand{\reducesTo}[1]{\overset{#1}{\rightarrow{}}}

\begin{proposition}\label{prop:compose_approx_homo}
    If $T_1 \reducesTo{\varepsilon_1} T_2$ and $T_2 \reducesTo{\varepsilon_2} T_3$, then $T_1 \reducesTo{\varepsilon_1 + \varepsilon_2} T_3$.
\end{proposition}

This proposition states the existence of the dashed arrow in the following diagram:

\[
\begin{tikzcd}[column sep=large]
T_1
  \arrow[r, "{\varepsilon_1}"]
  \arrow[rr, bend left=35, dashed, "{\varepsilon_1+\varepsilon_2}"]
&
T_2
  \arrow[r, "{\varepsilon_2}"]
&
T_3
\end{tikzcd}
\]

Theorem~\ref{teo:approximate_red_maintains_discounted} and Proposition~\ref{prop:compose_approx_homo} suggest that this notion of approximate homomorphism is natural and algebraically convenient. We recall that composability can be shown because we use total variation to compare the one-step dynamics: a different choice of distance to compare the distributions may not preserve this property. 

\subsection{The case of linear transducers}

To introduce an approximation error in the exact linear reduction we will equip each state vector space with a
norm, which we will use to measure the distance between different vectors (mainly, between the vector obtained through the reduction and the vectors from the transducer itself). Throughout this subsection, we assume that the output alphabet $O$ is finite.  If $V$
is a normed vector space, we denote its dual norm by $\|\cdot\|_{V^*}$. Given a linear operator $C$, we write $||C||_{a \to b}$ to denote the norm $\sup_{x : ||x||_a = 1}||Cx||_{b}$.

An arbitrary linear transducer does not necessarily implement an interface. In particular, there are some transducers for which the norm of the state vector tends to infinity as the transducer reads symbols. Such a behaviour troubles our notion of approximate reduction, since a small margin of error in the one-step dynamics can be amplified arbitrarily in the subsequent steps. Thus, to rule out this situation, we introduce the notion of \textit{contractive} transducer.

\begin{definition}
\label{def:symbolwise-contractive}
A linear transducer $G=(V,\xi,\lambda,\{M_{a,o}\}_{(a,o)\in A\times O})$
is \emph{contractive} if
\[
    \|\xi\|_V\leq 1,
    \qquad
    \|\lambda\|_{V^*}\leq 1,
    \qquad
    \|M_{a,o}x\|_V\leq \|x\|_V,
\]
for every $a\in A$, $o\in O$, and $x\in V$.
\end{definition}

The last condition ensures that after applying an evolution operator the norm of the vector state does not increase. Note that the canonical representation given by the Hankel matrix satisfies this definition.  Indeed, on the Hankel row space
$V_{\mathcal I}$, the prediction norm
\[
    \|r\|_{\mathrm{pred}}=\sup_{v\in\Sigma^*}|r(v)|
\]
makes every shift $R^{\mathcal I}_{a,o}$ nonexpansive, because
\[
    \|R^{\mathcal I}_{a,o}r\|_{\mathrm{pred}}
    =\sup_{v\in\Sigma^*}|r((a,o)v)|
    \leq \|r\|_{\mathrm{pred}}.
\]
Moreover, $\|\xi_{\mathcal I}\|_{\mathrm{pred}}=1$, while
$\|\lambda_{\mathcal I}\|_{V_{\mathcal I}^*}\leq1$ because $\lambda_{\mathcal I}(r)=r(\epsilon)$.

We now introduce our notion of approximate linear reduction.

\begin{definition}
\label{def:symbolwise-approx-linear-reduction}
Let $G=(V,\xi,\lambda,\{M_{a,o}\}_{a,o})$ and $G'=(W,\xi',\lambda',\{N_{a,o}\}_{a,o})$ be contractive linear transducers over the same input and output
alphabets, and let $\varepsilon\geq0$.  A bounded surjective linear map $L:V\to W$ is a
\emph{$\varepsilon$-linear reduction} from $G$ to $G'$ if
\begin{align*}
    \|L\xi-\xi'\|_W
        &\leq \varepsilon,\\
    \|LM_{a,o}x-N_{a,o}Lx\|_W
        &\leq \varepsilon\|x\|_V,\\
    \|\lambda'L-\lambda\|_{V^*}
        &\leq \varepsilon,
\end{align*}
for every $(a,o)\in A\times O$ and $x\in V$. We write $G\xrightarrow{\varepsilon}G'$ when such a map exists.
\end{definition}

When $G$ and $G'$ are contractive, the conditions for $\varepsilon=0$
are precisely the equations defining an exact linear reduction. This definition extends approximate
reductions between standard transducers.

\begin{proposition}
\label{prop:ordinary-to-linear-approx}
Let $T_1$ and $T_2$ be finite standard transducers over the same alphabets,
and suppose that a surjective state map $\phi:\states_1\to \states_2$ is an
$\varepsilon$-reduction. Then, if $G_1$ and $G_2$ are the linear implementations corresponding to $T_1$ and $T_2$ equipped with their $\ell_1$ norms, there exists a
$2\varepsilon$-linear reduction from $G_1$ to $G_2$ induced by $\phi$.
\end{proposition}

Approximate linear reductions can be composed in the same way as the standard approximate reductions. From now on, for a bounded linear map $L$,
write $c(L)=\max\{1,\|L\|\}$.
The following holds.

\begin{proposition}
\label{prop:linear-approx-composition}
Suppose that $G_0\xrightarrow{\varepsilon_1}G_1$ through the linear map $L$ and $G_1\xrightarrow{\varepsilon_2}G_2$ through the linear map $K$. Then $KL$ is a $c(K)\varepsilon_1+c(L)\varepsilon_2$ linear reduction from $G_0$ to $G_2$.
\end{proposition}

We next compare the interfaces implemented by approximately reduced linear
transducers. In the case of standard transducers we could prove in Theorem~\ref{teo:approximate_red_maintains_discounted} that the discounted metrics were preserved after an approximate reduction. For linear transducers we obtain a similar result, but with a weaker bound.

\begin{theorem}
\label{thm:linear-discounted-continuity}
Let $G$ and $G'$ be contractive linear transducers implementing
interfaces $\mathcal I_G$ and $\mathcal I_{G'}$.  If
\(
  G\xrightarrow{\varepsilon}G'
\), then, for every $n\geq1$,
\begin{equation}
\label{eq:linear-horizon-bound}
    \sup_{\textbf{a}\in \inputalph^n}
    \|D_{\mathcal I_G}(\textbf{a})-D_{\mathcal I_{G'}}(\textbf{a})\|_{\mathrm{TV}}
    \leq
    \min\left\{1,\frac{n+2}{2}|O|^n\varepsilon\right\}.
\end{equation}
Consequently, for every $\gamma\in(0,1)$,
\begin{equation}
\label{eq:linear-general-modulus}
    d_\gamma(\mathcal I_G,\mathcal I_{G'})
    \leq \sum_{n\geq0}\gamma^n
      \min\left\{1,\frac{n+2}{2}|O|^n\varepsilon\right\},
\end{equation}
which converges to 0 as $\varepsilon \to 0$.
\end{theorem}

Note that this implies that for small enough $\varepsilon$ both transducers implement a similar interface. The bound is somewhat weaker when compared to the one from Theorem~\ref{teo:approximate_red_maintains_discounted} because the notion of approximate reduction for standard transducers is stronger with respect to the one-step equivalence of the dynamics. For instance, Eq.~\eqref{eq:approx_kernel_pushforward} requires that the overall error (i.e. total variation) is bounded, while in Definition~\ref{def:symbolwise-approx-linear-reduction} we bound each error independently. This is the reason why a term $|\outputalph|$ shows up in the bound. We could fix this by changing the definition of approximate linear reduction, but it would require us to also modify the notion of contractive transducer. Moreover, the required change gives a notion of contractive transducer which does not include the canonical Hankel representations, which we want to use in later proofs.

Nonetheless, we want to highlight the fact that there are many other valid choices regarding these definitions. In our case, we wanted to prioritize the fact that our abstractions should extend the notion of approximate reduction for standard transducers (proven in Proposition~\ref{prop:ordinary-to-linear-approx}), should allow to represent the canonical Hankel constructions and should satisfy the simple and basic properties already seen for standard transducers (composability in Proposition~\ref{prop:linear-approx-composition} and continuity with regard to the discounted metrics in Theorem~\ref{thm:linear-discounted-continuity}).

\section{Approximate reductions between implementations of similar interfaces}\label{sec:approx_reduction_orders}

In this section we will study the set of transducers which implement a given interface $\interface$, and we will try to relate them through approximate homomorphisms. Moreover, we will look at the set of transducers implementing a \textit{similar} interface (using one of the distances for interfaces mentioned previously). Ideally, we would like for this set of transducers to share some property, since that would indicate an \textit{emergent} property related to the representation of the interfaces.

The following definition formalizes this set.

\begin{definition}
    Let $\interface$ be an interface, $\varepsilon\geq 0$ and $d$ some metric over the set of interfaces. We define
    \begin{align*}
        \mathcal{L}_{\interface}^{\varepsilon,d}
        =
        \left\{
            T : T \text{ is a transducer and }d(\interface_T,\interface)\leq \varepsilon
        \right\}.
    \end{align*}
    as the set of transducers that $\varepsilon$-approximate $\interface$.
\end{definition}

For every set $\mathcal{L}_{\interface{}}^{\varepsilon,d}$ we would like to understand whether there is some $T \in \mathcal{L}_{\interface{}}^{\varepsilon,d}$ such that, for every other $T'\in \mathcal{L}_{\interface{}}^{\varepsilon,d}$, it holds that $T'\overset{\delta}{\rightarrow} T$ for some small $\delta$, ideally scaling as $\delta = O(\varepsilon)$. We call such a transducer a $\delta$-minima of  $\mathcal{L}_{\interface}^{\varepsilon, d}$. We define
\begin{align*}
    \delta_{\varepsilon}(\interface) 
    =
    \inf\left\{
        \delta\in \mathbb{R}_{\geq 0}:
        \mathcal{L}_{\interface}^{\varepsilon,d}
        \text{ has a $\delta$-minima}
    \right\}.
\end{align*}
With this notation, our goal is to find bounds for $\delta_{\varepsilon}(\interface)$ in terms of $\varepsilon$. Is there a subset of interfaces which is well-behaved in this sense? Does it matter which type of transducers we consider? See Figure~\ref{fig:minimize_close} for a sketch of the type of behaviour that we aim for.

In the next subsections we will consider the set $\mathcal{L}_{\interface}^{\varepsilon,d}$ restricted to different types of transducers. To avoid cluttering the notation we won't add any more indices to this symbol, but rather take the convention that in each respective subsection this set is restricted to the set of transducers studied in the corresponding subsection.

\begin{figure}[ht]
    \centering
    \includegraphics[width=1\linewidth]{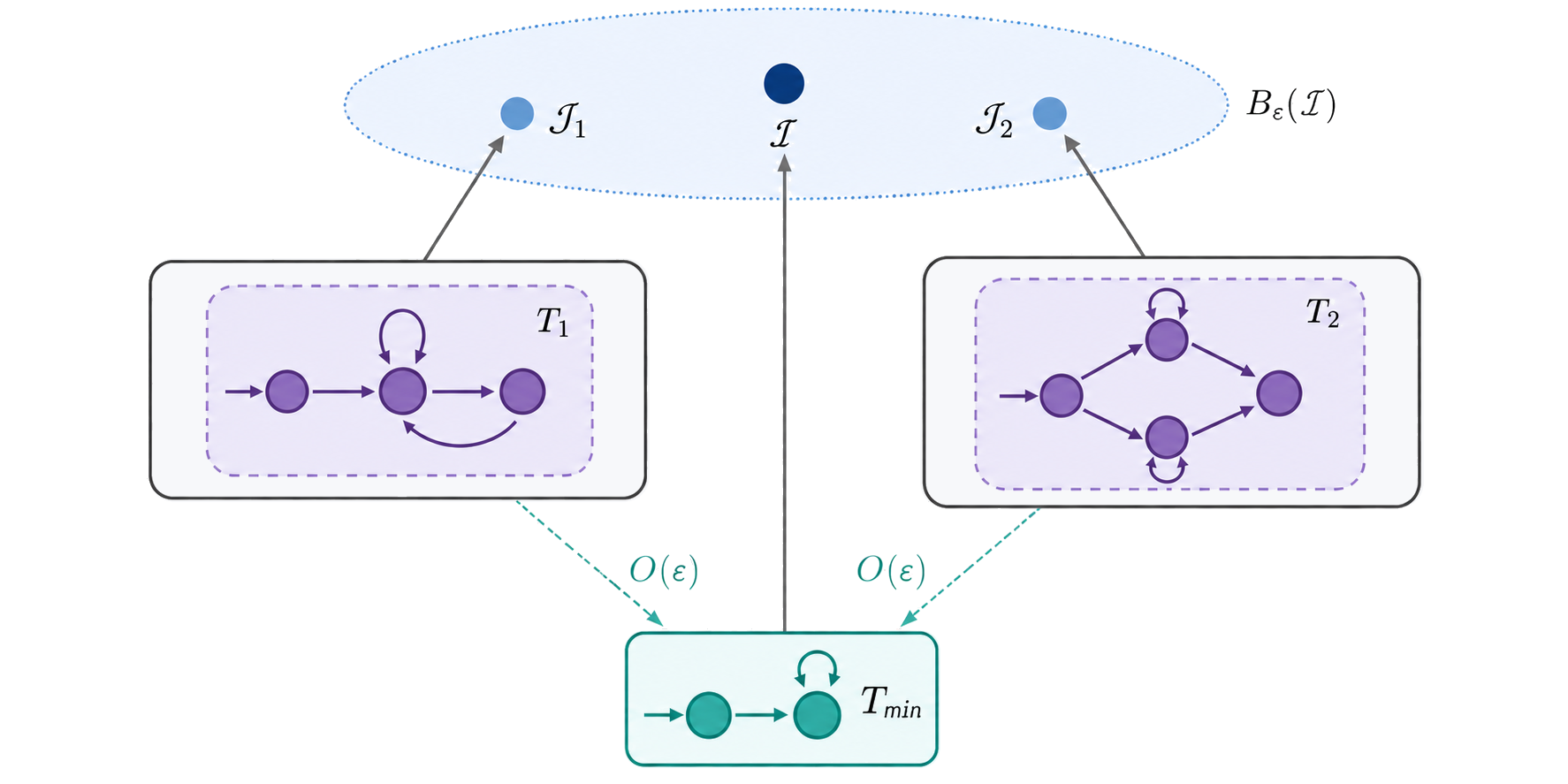}
    \caption{Schematic description of the type of convergence result we would like to prove. After fixing an interface of interest $\interface$, we look at all interfaces $\varepsilon$-close to $\interface$ in some distance. For each of these interfaces (such as $\interfaceJ_1$ and $\interfaceJ_2$) there are many transducers implementing the dynamics (respectively, $T_1$ and $T_2$). We say that there is a $O(\varepsilon)$-minimum if there is some transducer $T_{min}$ implementing an interface from $B_\varepsilon(\interface)$ such that for any transducer $T$ implementing an interface in $B_\varepsilon(\interface)$ it holds that $T \overset{O(\varepsilon)}{\rightarrow} T_{min}$. In the diagram this is represented by the transducer $T_{min}$ implementing the interface $\interface$, and there are $O(\varepsilon)$-reductions from both $T_1$ and $T_2$.}
    \label{fig:minimize_close}
\end{figure}

\subsection{Non-existence of \texorpdfstring{$\delta$}{delta}-minima for standard transducers}

As mentioned before, there exist an $\interface$ such that $\mathcal{L}_{\interface}^{0,d}$ does not have a 0-minimum when considering only standard transducers\footnote{Note that if $\varepsilon=0$ the choice of distance is irrelevant.}. Can this situation be avoided using $\delta$ reductions? Note that as $\delta$ increases we allow more reductions (in the limit, taking $\delta = 1$ allows every possible reduction), and thus it should make it simpler for convergent structures to arise.

In the next proposition we show that this is not the case.

\begin{proposition}
    \label{prop:no-exact-universal-target-below-one}
    There exists an interface $\interface$ such that, for every
    $\delta<1$, the set
    $\mathcal{L}_{\interface}^{0,d}$ does not have a $\delta$-minima.
\end{proposition}

Even though $\mathcal{L}_{\interface}^{0,d}$ does not have a common representation, it might be the case that when looking at approximate implementations of $\interface$ there is some convergent structure. Again, the answer is negative.
\begin{proposition}\label{prop:no-ordinary-local-minimization}
    Let $\interface$ be the interface from
    Proposition~\ref{prop:no-exact-universal-target-below-one}. Then, for every
    $\varepsilon\ge 0$, every transducer $C$, and every
    $\delta<1/2$, it is not true that $T'$ $\delta$-reduces to $C$ for every $T' \in \mathcal{L}_{\interface}^{\varepsilon,d_\infty}$. In particular, $\mathcal{L}_{\interface}^{\varepsilon, d_\infty}$ does not have a $\delta$-minima.
\end{proposition}

These two results show that, in the case of standard transducers, there are interfaces for which no convergent structure exists between the different implementations of the interface, at least when we formalize this structure through our local notion of approximate homomorphism. This result is robust even when nearby interfaces are considered. Moreover, the counterexample is simple (it is a low-dimensional finite-rank interface), and the result can be proven for other distances (such as the discounted one $d_{\gamma}$). Thus, we don't believe that there is a reasonable restricted set of interfaces for which we could bound $\delta_\varepsilon(\interface)$ by $O(\varepsilon)$. We remark that these results are an extension of the ones from~\cite{rosas2025ai} in the context of approximate homomorphisms and approximate implementations of interfaces. 

\subsection{Existence of \texorpdfstring{$\delta$}{delta}-minima for linear transducers}
\label{subsec:generalized-local-minima}

\providecommand{\interfaceJ}{\mathcal{J}}

We now show a positive result for linear transducers. We will show that for all finite-rank interfaces $\interface$ all nearby interfaces have a canonical representation which is similar to the one from $\interface$. This is intuitive: note that the canonical construction is induced by the rows of the Hankel matrix. If an interface is slightly perturbed then the Hankel matrix is slightly perturbed as well. Thus, we can map the new Hankel matrix to the original one identifying each row with the corresponding one from the original matrix.

To obtain the strongest result possible we will give a norm to the canonical representation which dominates the predictive one.

\begin{definition}
\label{def:hankel-atomic-norm}
Let $\interface$ be an interface and let
\(
 V_{\interface}=\operatorname{span}\{h_{\interface}(w):w\in\Sigma^*\}
\)
be its Hankel row space. For $r\in V_{\interface}$, define the atomic norm as
\begin{equation*}
    \|r\|_{\mathrm{at},\interface}
    :=
    \inf\left\{
       \sum_{i=1}^m|\alpha_i|:
       r=\sum_{i=1}^m\alpha_i h_{\interface}(w_i)
    \right\}.
\end{equation*}
We write $G_{\interface}^{\mathrm{at}}$ for the canonical Hankel implementation equipped
with this norm.
\end{definition}

We are using the word \textit{atom} to refer to each row of the Hankel matrix. We formally show that $||\cdot ||_{\mathrm{at}}$ is a norm.

\begin{lemma}\label{lem:atomic-contractivity}
    Let \(\interface\) be an interface and equip its Hankel row space \(V_\interface\) with
    the atomic norm.  Then \(\|\mathord{\cdot}\|_{\mathrm{at},\interface}\) is a norm and, for every \(r\in V_\interface\),
    \[
    \|r\|_{\mathrm{pred}}
    \le
    \|r\|_{\mathrm{at},\interface}.
    \]
    Moreover,
    \[
    \|\xi_\interface\|_{\mathrm{at},\interface}=1,
    \qquad
    \|\lambda_\interface\|_{(V_\interface,\|\mathord{\cdot}\|_{\mathrm{at},\interface})^*}=1,
    \]
    and every shift $R^\interface_\sigma$ is nonexpansive. Consequently, \(G_\interface^{\mathrm{at}}\) is a contractive linear
    transducer.
\end{lemma}

This norm measures what is the best way to write $r$ as a linear sum of the rows of the Hankel matrix, where we weight each sum with the sum of the absolute values of its coefficients.

Our notion of linear reduction requires the mapping to be surjective. Thus, to ensure this property we will look at invertible minors of the Hankel matrix. Let $\mathcal I$ be a finite-rank interface with $d=\operatorname{rank}(H_{\mathcal I})$. We can
choose prefixes $p_1,\ldots,p_d$ and suffixes $q_1,\ldots,q_d$ such that
\begin{equation}
\label{eq:chosen-hankel-minor}
    C_{\mathcal I}:=\bigl(F_{\mathcal I}(p_iq_j)\bigr)_{i,j=1}^d
\end{equation}
is invertible and $q_1=\epsilon$. Put
\[
    B_{\mathcal I}=
    \begin{pmatrix}
       h_{\mathcal I}(p_1)\\[-1mm]
       \vdots\\[-1mm]
       h_{\mathcal I}(p_d)
    \end{pmatrix},
    \qquad
    \Gamma_{\mathcal I}:=\|C_{\mathcal I}^{-1}\|_{\infty\to1}.
\]
For any interface $\mathcal J$, define analogously $C_{\mathcal J}:=\bigl(F_{\mathcal J}(p_iq_j)\bigr)_{i,j=1}^d$ and $\operatorname{ev}_\interfaceJ:V_\interfaceJ \to \mathbb{R}^d$ given by
\[
    \operatorname{ev}_\interfaceJ(r):=\bigl(r(q_1),\ldots,r(q_d)\bigr),
\]
. Finally, let $\Pi_{\mathcal J\to\mathcal I}:V_{\mathcal J}\longrightarrow V_{\mathcal I}$ be the linear map
\begin{equation*}
    \Pi_{\mathcal J\to\mathcal I}(r):=\operatorname{ev}_\interfaceJ(r)C_{\mathcal I}^{-1}B_{\mathcal I}.
\end{equation*}
The mapping $\Pi_{\mathcal J\to\mathcal I}$ translates $r \in V_{\interfaceJ}$ into a vector from $V_{\interface}$ by first evaluating the suffixes $\{q_i\}_{1\leq i \leq d}$, then doing a change of coordinates using $C_{\interface}$ and finally projecting the result into the rows from $V_\interface$ indexed by $\{p_i\}_{1\leq i \leq d}$.

The following lemma shows that this mapping commutes with the shift operators up to a small error with respect to the $||\cdot ||_{\mathrm{at}}$ norm if the interfaces are close. Moreover, whenever $C_{\interfaceJ}$ is invertible the mapping is surjective.

\begin{lemma}\label{lem:atomic-hankel-row-approximation}
    If $d_\infty(\mathcal I,\mathcal J)\leq\varepsilon$, then, for every $w\in\Sigma^*$,
    \begin{equation*}
        \|\Pi_{\mathcal J\to\mathcal I}h_{\mathcal J}(w)-h_{\mathcal I}(w)\|_{\mathrm{at},\mathcal I}
        \leq \Gamma_{\mathcal I}\varepsilon.
    \end{equation*}
    Moreover, if $C_{\mathcal J}$ is invertible, then $\Pi_{\mathcal J\to\mathcal I}$ is surjective.
\end{lemma}

It is a well-known fact that if a finite matrix $M$ is invertible, then adding a small amount of noise to $M$ keeps it invertible. We apply this observation to $C_{\interface}$ to guarantee that $C_{\interfaceJ{}}$ remains invertible. 

\begin{theorem}\label{thm:finite-rank-hankel-stability}
    Let $\mathcal I$ be a finite-rank interface and choose the minor $C_{\mathcal I}$ as in
    \eqref{eq:chosen-hankel-minor}, with $q_1=\epsilon$. Then, there exists $\overline\varepsilon(\mathcal I)>0$ such that, for every interface $\mathcal J$ satisfying
    \[
        d_\infty(\mathcal I,\mathcal J)\leq\varepsilon<\overline\varepsilon(\mathcal I),
    \]
    the map $\Pi_{\mathcal J\to\mathcal I}:G_{\mathcal J}^{\mathrm{at}}\to G_{\mathcal I}^{\mathrm{at}}$
    is a $2\Gamma_{\mathcal I}\varepsilon$-linear reduction.
\end{theorem}

From this theorem we get as an immediate corollary the existence of $\delta$-minima for a small enough neighbourhood of every finite-rank interface. 

\begin{corollary}\label{cor:finite-rank-atomic-minimum}
    For every finite-rank interface $\mathcal I$, its canonical realization $G_{\mathcal I}$, equipped
    with the atomic norm, is a $2\Gamma_{\mathcal I}\varepsilon$-minimum of
    $\mathcal L_{\mathcal I}^{\varepsilon,d_\infty}$ for every
    $0\leq\varepsilon<\overline\varepsilon(\mathcal I)$, when restricting $\mathcal L_{\mathcal I}^{\varepsilon,d_\infty}$ to contain only the canonical linear realizations with the atomic norm.
\end{corollary}

Note that our main theorem has to bound $\varepsilon$ to ensure that the reduction is surjective. Even though this requirement is reasonable (otherwise, we could reduce small transducers into subcomponents of bigger ones), in many applications it might make sense to ignore this restriction. That's why we phrased Theorem~\ref{thm:finite-rank-hankel-stability} in an independent way.

The results in this section are in some sense satisfactory: we observed that convergent structure (i.e. a $\delta$-minima for $\delta = O(\varepsilon)$) exists for every finite rank interface in a neighbourhood of the interface. Observe that in our statements we have to pick a norm and a distance in a somewhat arbitrary way. However, a similar result can be proven using the predictive norm. Conceptually, we believe that these result indicate that the convergent structure exists at the level of linear transducers even in the presence of perturbations in the implementations.

We remark that in Corollary~\ref{cor:finite-rank-atomic-minimum} we restrict the lattice to the minimal linear implementations. If we don't do this, we still can prove the existence of a reduction from any linear transducer $G$ in the set by composing the map $\rho_G$ from Lemma~\ref{lem:prediction-map} with the one from Theorem~\ref{thm:finite-rank-hankel-stability} whenever $\rho_G$ is bounded. Then, using Proposition~\ref{prop:linear-approx-composition} we would obtain an error that depends on $\|\rho_G\|$. However, there is no uniform bound on this value when using the atomic norm. There are other choices of norms which can solve this problem but they seem quite unnatural, and therefore we prefer to keep the corollary as stated, applying only to the lattice of ``optimal'' implementations.

\providecommand{\interfaceJ}{\mathcal{J}}

\subsection{Existence of \texorpdfstring{$\delta$}{delta}-minima for predictive transducers under a specific metric}
\label{subsec:predictive-instability}

As already mentioned, restricting to predictive transducers restores a canonical minimum for each fixed
interface, which we denote by $\mathsf E(\interface)$ (see Proposition~\ref{prop:epsilon_transducer_minimal}). This exact statement is not stable under the supremum metric.

\begin{proposition}\label{prop:predictive-rare-history-obstruction}
    There is an interface $\interface$ and a sequence of interfaces
    $(\interfaceJ_n)_{n\geq 1}$ such that
    \[
        d_\infty(\interface,\interfaceJ_n)=2^{-n}\longrightarrow 0,
    \]
    but every approximate reduction $\mathsf E(\interfaceJ_n)\reducesTo{\delta}\mathsf E(\interface)$ satisfies $\delta\geq 1/2$. Moreover, if a transducer $C$ receives
    $\delta$-reductions from both $\mathsf E(\interface)$ and
    $\mathsf E(\interfaceJ_n)$, then $\delta\geq 1/4$.
\end{proposition}

This counterexample also applies to the discounted metric, and we believe it highlights a limitation of predictive transducers. More precisely, if there is a history $h$ such that the interface $\interface$ conditioned on $h$ behaves in an extremely different way from the interface $\interfaceJ$ conditioned on $h$, then the local structure related to both predictive states will be different, even if $h$ is highly unlikely (and thus, it is ignored by most reasonable metrics). We now show that, as one would expect, this problem can be avoided by using precisely a notion of distance between interfaces that values every possible conditioning independently of its probability.

For a positive probability history $h$, let $\interface^h$ denote the residual interface after conditioning on $h$ (see Eq.~\eqref{eq:residual_interface} for the exact definition). Write
$\operatorname{supp}(\interface)$ for the set of positive probability histories and define
\[
    d_{\mathrm{res}}(\interface,\interfaceJ)
    =
    \begin{cases}
        \displaystyle
        \sup_{h\in\operatorname{supp}(\interface)}
        d_\infty(\interface^h,\interfaceJ^h),
        &\operatorname{supp}(\interface)=\operatorname{supp}(\interfaceJ),\\[1.2ex]
        1,
        &\text{otherwise.}
    \end{cases}
\]
This metric compares the predictive laws after every history, including the ones with low probability. Let
$[h]_{\interface}$ denote the state of $\mathsf E(\interface)$ reached after $h$, and put
\[
    \Delta_{\interface}
    =
    \inf\left\{
        d_\infty(\interface^h,\interface^u):
        h,u \in \operatorname{supp}(\interface), [h]_{\interface}\neq [u]_{\interface}
    \right\}.
\]
We use the convention $\Delta_{\interface}=+\infty$ if there is only one predictive state.
If $\mathsf E(\interface)$ is finite, then $\Delta_{\interface}>0$. Intuitively, $\Delta_{\interface}$ is a lower bound on the difference between the interfaces induced by each of the predictive states.

If there is an interface $\interfaceJ$ such that $d_{\mathrm{res}}(\interface,\interfaceJ)$ is much smaller than $\Delta_\interface$ then the predictive transducers implementing $\interface$ and $\interfaceJ$ will be similar. We formalize this in our last theorem.

\begin{theorem}\label{thm:predictive-residual-stability}
    Let \(\interface\) and \(\interfaceJ\) be interfaces over the same input and output
    alphabets.  Suppose that \(\Delta_\interface>0\) and that
    \[
    d_{\mathrm{res}}(\interface,\interfaceJ)
    \leq \varepsilon
    <
    \min\left\{1,\frac{\Delta_\interface}{2}\right\}.
    \]
    Then the assignment
    \[
    \phi([h]_\interfaceJ):=[h]_\interface,
    \]
    is a well-defined surjective state map and determines an
    \(\varepsilon\)-reduction $E(\interfaceJ)\longrightarrow E(\interface)$.
    Consequently, every predictive transducer implementing \(\interfaceJ\) admits
    an \(\varepsilon\)-reduction to \(E(\interface)\).
\end{theorem}

The intuition for the condition $\varepsilon < \frac{\Delta_\interface}{2}$ comes from the triangle inequality: if we want to prove that $\phi$ is well defined, we need to ensure that if $h_1 \sim h_2$ in $\interfaceJ$ then $h_1 \sim h_2$ in $\interface$. This is equivalent to verifying that $\interface^{h_1} = \interface^{h_2}$, and by triangle inequality we see that
\begin{align*}
    d_\infty(\interface^{h_1},\interface^{h_2})
    &\leq
    d_\infty(\interface^{h_1},\interfaceJ^{h_1})
    +
    d_\infty(\interfaceJ^{h_1},\interfaceJ^{h_2})
    +
    d_\infty(\interfaceJ^{h_2},\interface^{h_2})
    \\
    &\leq 2\varepsilon
    <\Delta_\interface.
\end{align*}
Then, by definition of $\Delta_\interface$ it must be the case that $\interface^{h_1} = \interface^{h_2}$.

As a corollary, we obtain the existence of $\delta$-minimums for the set of predictive transducers.

\begin{corollary}\label{coro:existence_delta_min_predictive}
    Let $\interface$ be an interface such that $\Delta_{\interface} > 0$. Then, for every $0 < \varepsilon < \min\left\{1,\frac{\Delta_\interface}{2}\right\}$ and every predictive transducer $T$ implementing an interface $\interfaceJ$ with $d_{\mathrm{res}}(\interface, \interfaceJ) \leq \varepsilon$ there is a $\varepsilon$-reduction from $T$ to $\mathsf E(\interface)$.
\end{corollary}

This is a positive result regarding convergent structure, but its application is restricted to scenarios where the residual distance makes sense. In the following we describe an example of a situation in which two different stochastic systems can induce interfaces which are close according to $d_{\mathrm{res}}$.


\begin{example}
    \normalfont Let $\states=\outputalph=\{0,1,2\}$ and $\inputalph = \{\star\}$ (i.e. the dynamics are actionless because there is a single action). and consider the transition matrix
    \[
        P=
        \begin{pmatrix}
            1/2 & 1/4 & 1/4\\
            1/4 & 1/2 & 1/4\\
            1/4 & 1/4 & 1/2
        \end{pmatrix}.
    \]
    Let state $i$ output $i$ before making a transition. Thus,
    the transition kernel is
    \[
        \kappa(j,o\mid i)
        =
        P(i,j)\mathbf 1_{\{o=i\}}.
    \]
    Consider two transducers $T$ and $T_\varepsilon$ using the kernel
    $\kappa$, but with different initial distributions given by $p=
        \left(\frac13,\frac13,\frac13\right)$ for $T$ and $p_\varepsilon=
        \left(
            \frac13+\varepsilon,
            \frac13-\varepsilon,
            \frac13
        \right)$ for $T_\varepsilon$. See Figure~\ref{fig:ergodic-dres-example} for a visual description.
    
    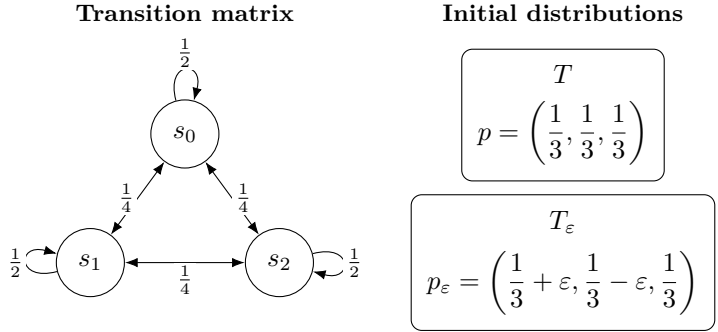
\begin{figure}[t]
        \centering
        \begin{tikzpicture}[
            >=Latex,
            state/.style={
                circle,
                draw,
                minimum size=9mm,
                inner sep=0pt
            },
            transition/.style={<->,thin},
            self/.style={->,thin},
            model/.style={
                draw,
                rounded corners,
                inner sep=6pt,
                align=center
            },
            edgeLabel/.style={
                fill=white,
                inner sep=1pt,
                font=\small
            }
        ]
            \node[font=\small\bfseries] at (0,2)
                {Transition matrix};
    
            \node[state] (s0) at (0,0.4)      {$s_0$};
            \node[state] (s1) at (-1.25,-1.3) {$s_1$};
            \node[state] (s2) at (1.25,-1.3)  {$s_2$};
    
            \draw[transition]
                (s0) -- node[edgeLabel,left] {$\frac14$} (s1);
            \draw[transition]
                (s0) -- node[edgeLabel,right] {$\frac14$} (s2);
            \draw[transition]
                (s1) -- node[edgeLabel,below] {$\frac14$} (s2);
    
            \path[self]
                (s0) edge[loop above,looseness=6]
                    node[edgeLabel,above] {$\frac12$} (s0)
                (s1) edge[loop left,looseness=6]
                    node[edgeLabel,left] {$\frac12$} (s1)
                (s2) edge[loop right,looseness=6]
                    node[edgeLabel,right] {$\frac12$} (s2);

            \node[font=\small\bfseries] at (5,2)
                {Initial distributions};
    
            \node[model] at (5,0.65) {
                $T$\\[2mm]
                $\displaystyle
                p=\left(\frac13,\frac13,\frac13\right)$
            };
    
            \node[model] at (5,-1.3) {
                $T_\varepsilon$\\[2mm]
                $\displaystyle
                p_\varepsilon=
                \left(
                    \frac13+\varepsilon,
                    \frac13-\varepsilon,
                    \frac13
                \right)$
            };
        \end{tikzpicture}
    
        \caption{We consider two transducers $T$ and $T_\varepsilon$ with the same transition kernel described by the Markov process on the left but different initial distributions given on the right side. It can be seen that for small enough $\varepsilon > 0$ the induced interfaces satisfy $d_{\mathrm{res}}(\interface_T, \interface_{T_\varepsilon}) \leq \varepsilon < \Delta_\interface/2$.}
        \label{fig:ergodic-dres-example}
    \end{figure}
    It can be seen that $d_{\mathrm{res}}(\interface_T, \interface_{T_\varepsilon}) \leq \varepsilon$: note that conditioning on any non-empty history $h$ the interfaces $\interface^h_T$ and $\interface^h_{T_\varepsilon}$ are equal, and moreover $d_\infty(\interface_T, \interface_{T_\varepsilon}) \leq \varepsilon$. Moreover, it can be seen that $\Delta_{\interface_T} = \frac{1}{6}$. Thus, due to Theorem~\ref{thm:predictive-residual-stability} we conclude that if $\varepsilon < \frac{1}{12}$ it holds that the minimal predictive implementation of $\interface_T$ and of $\interface_{T_\varepsilon}$ are $\varepsilon$-close through an approximate reduction.
\end{example}

More generally, different processes with the same transition dynamics but different initial distributions can be close in the $d_{\mathrm{res}}$ metric whenever the effect of conditioning ensures that the current state for both processes is the same.

\section{Conclusion}\label{sec:conclusion}

\textbf{Summary.} In this work we looked for theoretical evidence supporting the empirical observation that different neural network models, sometimes even supported on different architectures, tend to converge to similar representations in their internal layers. To investigate this idea we proposed the approach described in Figure~\ref{fig:diagram_world_model_approach}: we assumed that the reasoning inside the internal layers can be represented through some abstraction, and then we tried to prove some convergence at the level of these objects. In our case, we considered transducers to capture these world models, and to find convergent structure we looked for homomorphisms between them. Previous work had already proven that for the case of linear and predictive transducers there always exists a minimal transducer implementing a given dynamics, and that for all other non-minimal transducers there is a homomorphism to this minimal one \cite{rosas2025ai}. In this paper we improved this result by showing that the existence of such an homomorphism remains even when we consider transducers that do not implement \textit{exactly} the same dynamics.

To do this, we first introduced a notion of approximate homomorphism for standard transducers (Definition~\ref{def:trans_approx_hom}) as well as for the linear ones (Definition~\ref{def:symbolwise-contractive}). We showed that, although there are many ways to define such a family of homomorphisms, the ones proposed here have good algebraic properties: they preserve the dynamics under discounted metrics (Theorems~\ref{teo:approximate_red_maintains_discounted} and~\ref{thm:linear-discounted-continuity}), are composable (Propositions~\ref{prop:compose_approx_homo} and~\ref{prop:linear-approx-composition}), and the linear approximate homomorphism is a direct extension of the standard one (Proposition~\ref{prop:ordinary-to-linear-approx}).

With these tools developed, we looked for convergence theorems in the approximate setting: we looked for conditions under which all transducers implementing similar dynamics share some common structure, which we aimed to capture through approximate homomorphisms (see Figure~\ref{fig:minimize_close} for a visual sketch of the idea). In particular, we showed that (1) For standard transducers, this type of convergence seems to not be possible, even when considering simple dynamics (Propositions~\ref{prop:no-exact-universal-target-below-one} and~\ref{prop:no-ordinary-local-minimization}), (2) For linear transducers, simple enough dynamics (more technically, finite-rank interfaces) always admit a type of convergence between all the minimal linear implementations of $\varepsilon$-close dynamics (Theorem~\ref{thm:finite-rank-hankel-stability}), and, finally, that (3) For predictive transducer, there is a metric such that all transducer implementing dynamics close enough according to this metric will share structural properties between them (Theorem~\ref{thm:predictive-residual-stability}).

We believe these are positive theoretical results regarding the existence of convergent structure, in the context of both both linear and predictive transducers. The former family seems to be the natural model for capturing latent representation in modern neural networks, considering especially that the model parameters live in a vector space. The latter one, although less natural, still has been seen to show up inside the residual stream of transformers \cite{shai2026transformers,shai2024transformers}, and thus understanding these representational properties might shed light into the behaviour of modern LLMs.

\textbf{Limitations.} Through the development of this work we discovered that there are many ways to formalize approximate convergence in the context of transducers. Although the proposals here satisfy good properties and extend previous ones~\cite{ravindran2003smdp} there is still room for developing a more general theory of approximate homomorphisms. Moreover, in the context of linear transducers we had to equip the underlying vector space with a norm to measure distance between vectors, thus introducing another ``parameter'' to our theory. Although many of our main theorems can be proven for other choices of norms and distances (such as Theorem~\ref{thm:finite-rank-hankel-stability}), it would be great to have a more robust understanding of the precise hypothesis required to conclude structural convergence in the approximate setting. 

\textbf{Future work.} We describe some future lines of work starting from the developments in this paper.

\begin{itemize}
    \item \textbf{Experimental validation}: these results give predictions on the structural convergence of deep neural networks under the hypothesis that they use transducers in their latent space. In particular, for models trained on similar data it must be the case that their internal representation can be translated with a linear map (such type of translation scheme is usually referred to as ``stitching'', and has been studied in the literature~\cite{bansal2021revisiting,chen2026transferring}). To validate these hypotheses, it would be interesting to train modern models with data generated from specific linear transducers and then see whether these transducers can be found in the learned representations. To do this, one could reproduce the setting from~\cite{shai2026transformers,shai2024transformers}. 
    \item \textbf{Poset structure}: the original goal of this project was to study the poset of approximate homomorphism between world models. More precisely, we would like to understand how this poset looks like when we order it through the relation induced by the existence of an approximate homomorphism. A central and simple question is: under which hypothesis can we guarantee that this poset has cut-points, in the sense of intermediate models $M$ such that, for any other model $M'$, it holds that $M$ has an homomorphism to $M'$ or the other way around. Given the introduced notions of approximate homomorphisms, this question can now be approached in the context of transducers.
    \item \textbf{Improved abstractions}: Our abstractions are still limited and do not  represent the myriad of forms on which ``abstraction'' and ``reasoning'' can occur inside modern AI models. Two recognizable improvement would be (1) introducing some non-linearity in the notion of linear transducers, to model the effect of the activation functions between layers, and (2) introducing a global error inside the notion of approximate homomorphism to allow homomorphisms that preserve the local structure of most of the states but fail completely in a small subset (such a notion would capture more faithfully what happens during model stitching when the target network is a bigger model than the source network).
\end{itemize}

\section{Acknowledgements}

This work was funded by the Advanced Research + Invention Agency (ARIA) through project code MSAI-SE01-P005. We would like to thank the Dovetail Research team\footnote{\url{https://dovetailresearch.org/}.} for comments and suggestions on the draft of this paper, and we are especially grateful to Alex Altair, Alfred Harwood, Jose Faustino and Neal Batra for fruitful discussions.

\textbf{AI disclosure}: We used ChatGPT 5.5 and 5.6 for proofreading, creating diagrams, writing down simple proofs (such as the proof from Lemma~\ref{lem:atomic-contractivity}) and quickly exploring variants of the results (such as checking whether the proof of Theorem~\ref{thm:finite-rank-hankel-stability} holds for other choices of distances and norms). All content created by AI was revised and rewritten to improve readability and clarity of exposition. 

\appendix

\appendix

\section{Appendix}

\subsection{Comparison of notions of exact homomorphisms}\label{apx:comparison_homo}

We describe the original formulation of homomorphism and see how it differs from ours.

\begin{definition}[Homomorphism from \cite{rosas2025ai}]\label{def:original_conditional_homomorphism}
    Given two transducers $T_1 = (\states_1, \inputalph_1, \outputalph_1, \kappa_1, p_1)$ and $T_2 = (\states_2, \inputalph_2, \outputalph_2, \kappa_2, p_2)$, a homomorphism is given by three mappings $\langle \phi:\states_1\to\states_2,f:\inputalph_1\to\inputalph_2,g:\outputalph_1\to\outputalph_2\rangle$
    satisfying~\eqref{eq:exact_initial_pushforward}, the condition
    \begin{align}\label{eq:original_output_pushforward}
        \Pr{}_{T_2}(o_2|\phi(s_1),f(a_1))
        =
        \sum_{o_1\in g^{-1}(o_2)}
        \Pr{}_{T_1}(o_1|s_1,a_1).
    \end{align}
    for every $s_1 \in \states_1, a_1 \in \inputalph_1$ and $o_2 \in \outputalph_2$, and whenever $\Pr{}_{T_1}(o_1| s_1,a_1) > 0$ the condition
    \begin{align}\label{eq:original_conditional_transition}
        \Pr{}_{T_2}(s_2|\phi(s_1),f(a_1),g(o_1))
        =
        \sum_{s'\in\phi^{-1}(s_2)}
        \Pr{}_{T_1}(s'|s_1,a_1,o_1)
    \end{align}
    for every $o_1 \in \outputalph_1$ and $s_2\in\states_2$.
\end{definition}

We now show that this definition and ours coincide when $g$ is injective.

\begin{proposition}\label{prop:equiv_conditional_homomorphism}
    Definitions~\ref{def:transducer_homomorphism} and~\ref{def:original_conditional_homomorphism} coincide if $g$ is injective.
\end{proposition}

\begin{proof}
    Let $T_1 = (\states_1, \inputalph_1, \outputalph_1, \kappa_1, p_1)$ and $T_2 = (\states_2, \inputalph_2, \outputalph_2, \kappa_2, p_2)$ be transducers, and let $\langle \phi, f, g\rangle$ be maps $\phi:\states_1 \to \states_2$, $f:\inputalph_1 \to \inputalph_2$ and $g:\outputalph_1 \to \outputalph_2$. We will prove that this tuple satisfies Definition~\ref{def:transducer_homomorphism} if and only if it satisfies Definition~\ref{def:original_conditional_homomorphism}.

    Assume first that the joint-kernel condition~\eqref{eq:exact_kernel_pushforward} holds. Summing both sides over $s_2\in\states_2$ gives
    \begin{align*}
        \Pr{}_{T_2}(o_2|\phi(s_1),f(a_1))
        &=
        \sum_{s_2\in\states_2}\kappa_2(s_2,o_2|\phi(s_1),f(a_1))\\
        &=
        \sum_{s_2\in\states_2}
        \sum_{\substack{s'\in\phi^{-1}(s_2)\\ o'\in g^{-1}(o_2)}}
        \kappa_1(s',o'|s_1,a_1)\\
        &=
        \sum_{o'\in g^{-1}(o_2)}
        \Pr{}_{T_1}(o'|s_1,a_1),
    \end{align*}
    which is exactly~\eqref{eq:original_output_pushforward}. Now fix $o_1\in\outputalph_1$ such that $\Pr{}_{T_1}(o_1|s_1,a_1)>0$. Since $g$ is injective, $g^{-1}(g(o_1))=\{o_1\}$. Hence
    \begin{align*}
        \Pr{}_{T_2}(s_2|\phi(s_1),f(a_1),g(o_1))
        &=
        \frac{\kappa_2(s_2,g(o_1)|\phi(s_1),f(a_1))}
        {\Pr{}_{T_2}(g(o_1)|\phi(s_1),f(a_1))}\\
        &=
        \frac{\sum_{s'\in\phi^{-1}(s_2)}\kappa_1(s',o_1|s_1,a_1)}
        {\Pr{}_{T_1}(o_1|s_1,a_1)}\\
        &=
        \sum_{s'\in\phi^{-1}(s_2)}
        \Pr{}_{T_1}(s'|s_1,a_1,o_1).
    \end{align*}
    Thus the original conditional formulation follows. The initial condition is the same in both definitions.

    Conversely, assume Definition~\ref{def:original_conditional_homomorphism}. We prove~\eqref{eq:exact_kernel_pushforward}. If $o_2\notin g(\outputalph_1)$, then the right-hand side of~\eqref{eq:original_output_pushforward} is zero. Therefore $\Pr{}_{T_2}(o_2|\phi(s_1),f(a_1))=0$, and then $\kappa_2(s_2,o_2|\phi(s_1),f(a_1))=0$ for every $s_2$.

    It remains to consider $o_2\in g(\outputalph_1)$. By injectivity there is a unique $o_1\in\outputalph_1$ such that $g(o_1)=o_2$. If $\Pr{}_{T_1}(o_1|s_1,a_1)>0$, multiplying~\eqref{eq:original_conditional_transition} by~\eqref{eq:original_output_pushforward} gives
    \begin{align*}
        \kappa_2(s_2,o_2|\phi(s_1),f(a_1))
        &=
        \Pr{}_{T_2}(o_2|\phi(s_1),f(a_1))
        \Pr{}_{T_2}(s_2|\phi(s_1),f(a_1),o_2)\\
        &=
        \Pr{}_{T_1}(o_1|s_1,a_1)
        \sum_{s'\in\phi^{-1}(s_2)}
        \Pr{}_{T_1}(s'|s_1,a_1,o_1)\\
        &=
        \sum_{s'\in\phi^{-1}(s_2)}
        \kappa_1(s',o_1|s_1,a_1)\\
        &=
        \sum_{\substack{s'\in\phi^{-1}(s_2)\\ o'\in g^{-1}(o_2)}}
        \kappa_1(s',o'|s_1,a_1).
    \end{align*}
    If $\Pr{}_{T_1}(o_1|s_1,a_1)=0$, then~\eqref{eq:original_output_pushforward} gives $\Pr{}_{T_2}(o_2|\phi(s_1),f(a_1))=0$. Again nonnegativity forces both sides of~\eqref{eq:exact_kernel_pushforward} to be zero. The kernel condition therefore holds in all cases, and the initial condition is shared by the two definitions.
\end{proof}

When $g$ is not injective, the two formulations need not agree. In the original definition, the only outputs that can be coarse-grained are those with the same output laws for each state. Meanwhile, our formulation only requires equality after averaging over the whole fiber $g^{-1}(o)$ and $\phi^{-1}(s)$. As already mentioned, since we focus on reductions this distinction is irrelevant.

\subsection{Total variation}\label{subsec:tv_preliminaries}

We collect here some relevant facts about total variation. All probability spaces in the sequel are finite or countable.

\begin{definition}[Total variation]\label{def:total_variation}
    Let $\mu,\nu\in \Delta(X)$. Their total variation distance is
    \[
        \|\mu-\nu\|_{\mathrm{TV}}
        =
        \sup_{A\subseteq X}|\mu(A)-\nu(A)|
        =
        \frac12\sum_{x\in X}|\mu(x)-\nu(x)|.
    \]
\end{definition}

If $h:X\to Y$ is a map and $\mu\in\Delta(X)$, we write $h_*\mu\in\Delta(Y)$ for the push-forward distribution,
\[
    (h_*\mu)(y)=\sum_{x\in h^{-1}(y)}\mu(x).
\]
If $K:X\to \Delta(Y)$ is a Markov kernel and $\mu\in\Delta(X)$, we write $\mu K\in\Delta(Y)$ for the distribution
\[
    (\mu K)(y)=\sum_{x\in X}\mu(x)K(y|x).
\]

\begin{lemma}\label{lem:tv_basic_properties}
    Let $\mu,\nu\in\Delta(X)$, let $h:X\to Y$ and $r:Y\to Z$ be maps, and let $K,L:X\to \Delta(Y)$ be Markov kernels. Then, the following properties hold:
    \begin{enumerate}
        \item \textbf{Convexity.} If $\lambda_i\geq 0$, $\sum_i\lambda_i=1$, and $\mu_i,\nu_i\in\Delta(X)$, then
        \[
            \left\|\sum_i\lambda_i\mu_i-\sum_i\lambda_i\nu_i\right\|_{\mathrm{TV}}
            \leq
            \sum_i\lambda_i\|\mu_i-\nu_i\|_{\mathrm{TV}}.
        \]
        \item \textbf{Push-forwards compose.}
        \[
            (r\circ h)_*\mu=r_*(h_*\mu).
        \]
        \item \textbf{Push-forward contraction.}
        \[
            \|h_*\mu-h_*\nu\|_{\mathrm{TV}}
            \leq
            \|\mu-\nu\|_{\mathrm{TV}}.
        \]
        In particular, marginalization contracts total-variation.
        \item \textbf{Kernel contraction.}
        \[
            \|\mu K-\nu K\|_{\mathrm{TV}}
            \leq
            \|\mu-\nu\|_{\mathrm{TV}}.
        \]
        \item \textbf{Kernel perturbation bound.} If
        \[
            \sup_{x\in X}\|K(\cdot|x)-L(\cdot|x)\|_{\mathrm{TV}}\leq \varepsilon,
        \]
        then
        \[
            \|\mu K-\nu L\|_{\mathrm{TV}}
            \leq
            \|\mu-\nu\|_{\mathrm{TV}}+\varepsilon.
        \]
    \end{enumerate}
\end{lemma}

\begin{proof}
    Convexity follows directly from the $\ell^1$ expression for total variation and the triangle inequality:
    \begin{align*}
        \left\|\sum_i \lambda_i \mu_i - \sum_i \lambda_i \nu_i\right\|_{TV}&=\frac12\sum_x\left|\sum_i\lambda_i(\mu_i(x)-\nu_i(x))\right|\\
        &\leq
        \sum_i\lambda_i\frac12\sum_x|\mu_i(x)-\nu_i(x)|\\
        &= \sum_i \lambda_i \|\mu_i - \nu_i\|_{\mathrm{TV}}.
    \end{align*}
    The composition identity is immediate from the definition of push-forward. For push-forward contraction, use the supremum characterization:
    \[
        \|h_*\mu-h_*\nu\|_{\mathrm{TV}}
        =
        \sup_{B\subseteq Y}|\mu(h^{-1}(B))-\nu(h^{-1}(B))|
        \leq
        \sup_{A\subseteq X}|\mu(A)-\nu(A)|
        =
        \|\mu-\nu\|_{\mathrm{TV}}.
    \]
    Kernel contraction follows from the $\ell^1$ expression:
    \begin{align*}
        \|\mu K-\nu K\|_{\mathrm{TV}}
        &=
        \frac12\sum_y\left|\sum_x(\mu(x)-\nu(x))K(y|x)\right|\\
        &\leq
        \frac12\sum_x|\mu(x)-\nu(x)|\sum_yK(y|x)
        =
        \|\mu-\nu\|_{\mathrm{TV}}.
    \end{align*}
    Finally,
    \[
        \|\mu K-\nu L\|_{\mathrm{TV}}
        \leq
        \|\mu K-\nu K\|_{\mathrm{TV}}+
        \|\nu K-\nu L\|_{\mathrm{TV}}.
    \]
    The first term is bounded by $\|\mu-\nu\|_{\mathrm{TV}}$ by kernel contraction, and the second by
    \[
        \sum_x\nu(x)\|K(\cdot|x)-L(\cdot|x)\|_{\mathrm{TV}}
        \leq \varepsilon
    \]
    by convexity. This proves the perturbation bound.
\end{proof}

\subsection{Deferred proofs}

\begin{proof}[Proof of Lemma~\ref{lem:prediction-map}]
    We first check that $\rho_G$ is well-defined. Suppose that
    $\sum_i \alpha_i M_{w_i}\xi=0$. Then, for every suffix $v\in\Sigma^*$,
    \[
        \sum_i \alpha_i h_\interface(w_i)(v)
        =
        \sum_i \alpha_i F_\interface(w_i v)
        =
        \sum_i \alpha_i \lambda(M_vM_{w_i}\xi)
        =
        \lambda\!\left(M_v\sum_i \alpha_iM_{w_i}\xi\right)
        =0.
    \]
    Hence $\sum_i\alpha_i h_\interface(w_i)=0$, so the assignment is well-defined. Linearity is immediate from the definition.
    It is surjective because $V_\interface$ is spanned by the rows
    $h_\interface(w)$. The identities $\rho_G\xi=\xi_\interface$, $\rho_G M_\sigma=R^\interface_\sigma\rho_G$ and $\lambda_\interface\rho_G=\lambda$
    follow directly.
\end{proof}

\begin{proof}[Proof of Proposition~\ref{prop:epsilon_transducer_minimal}]
    Starting from $[\epsilon]_{\interface}$, Eq.~\eqref{eq:epsilon_kernel} reproduces the conditional output probabilities of $\interface$ after every admissible history. Hence $E(\interface)$ implements $\interface$. Its state after observing $h$ is $[h]_{\interface}$, so the residual future law is determined by the current state. Thus, it is predictive.
    
    Now let $T$ be a predictive transducer implementing $\interface$. For each reachable $s\in\states$, choose an admissible history $h$ such that $q_T(s\mid h)>0$ and define
    \begin{align}
        \phi_T(s)=[h]_{\interface}.
        \label{eq:predictive_reduction_map}
    \end{align}
    This is well defined: if both $h$ and $h'$ are compatible with $s$, predictivity gives
    $\interface^h=\interface_{T,s}=\interface^{h'}$. It is surjective because every admissible history has at least one state in the support of its posterior. Also, every state with positive initial probability is compatible with the empty history, so
    $(\phi_T)_*p=\delta_{[\epsilon]_{\interface}}$.
    
    Fix $s$ and choose a compatible history $h$. Predictivity gives
    $\Pr{}_T(o\mid s,a)=\Pr{}_{\interface^h}(o\mid a)$. Moreover, whenever
    $\kappa(s',o\mid s,a)>0$, the state $s'$ is compatible with the extended history $h(a,o)$, and therefore
    $\phi_T(s')=[h(a,o)]_{\interface}$. Thus all the probability mass associated with output $o$ is pushed forward to the unique state prescribed by Eq.~\eqref{eq:epsilon_kernel}, and
    \[
        (\phi_T\times\idfun)_*\kappa(\cdot,\cdot\mid s,a)
        =
        \kappa_{\epsilon}(\cdot,\cdot\mid\phi_T(s),a).
    \]
    Hence $\phi_T$ is a reduction.
\end{proof}

\begin{proof}[Proof of Theorem~\ref{teo:approximate_red_maintains_discounted}]
    Let \(\phi:\states_1\to \states_2\) be the map of the
    \(\varepsilon\)-reduction. Since a reduction preserves the input
    and output alphabets, let's write the common alphabets of both transducers as \(\inputalph\) and \(\outputalph\).
    Fix a section \(r:\states_2\to \states_1\) of \(\phi\), so that
    \(\phi(r(u))=u\) for every \(u\in \states_2\).
    
    For \(\mathbf a\in \inputalph^n\), define two probability distributions
    \(P_{\mathbf a}\) and \(Q_{\mathbf a}\) on \(\outputalph^n\times \states_2\) by
    \begin{align*}
    P_{\mathbf a}(\mathbf o,u)
    &=
    \sum_{s\in\phi^{-1}(u)}
     \Pr{}_{T_1}\!
     \left(
       O_{1:n}=\mathbf o,\,
       S_n^1=s
       \mid A_{1:n}=\mathbf a
     \right),
    \\
    Q_{\mathbf a}(\mathbf o,u)
    &=
     \Pr{}_{T_2}\!
     \left(
       O_{1:n}=\mathbf o,\,
       S_n^2=u
       \mid A_{1:n}=\mathbf a
     \right),
    \end{align*}
    where $S_n^i$ is a random variable denoting the state of the transducer $T_i$ at step $n$, $O_{1:n}$ denotes the first $n$ observed outputs and $A_{1:n}$ the first $n$ inputs. Thus, \(P_{\mathbf a}\) is the joint law of the output prefix and the
    coarse-grained state \(\phi(S_n^1)\) under \(T_1\) assuming inputs $\textbf{a}$, whereas
    \(Q_{\mathbf a}\) is the corresponding joint law under \(T_2\). We prove by induction on \(n\) that, for every
    \(\mathbf a\in \inputalph^n\),
    \begin{equation}
    \label{eq:repair-joint-law-induction}
     \|P_{\mathbf a}-Q_{\mathbf a}\|_{\mathrm{TV}}
     \le (n+1)\varepsilon.
    \end{equation}
    This is intuitive: initially the two distribution differ by at most $\varepsilon$ because of the error in the initial distribution, and after each step this error increases by at most $\varepsilon$ because the one-step transitions between $T_1$ and $T_2$ differ locally (i.e. when comparing $s\in \states_1$ with $\phi(s)$) by at most $\varepsilon$
    
    For \(n=0\), the output prefix is empty and
    \[
     P_{\epsilon}(\epsilon,u)=(\phi_*p_1)(u),
     \qquad
     Q_{\epsilon}(\epsilon,u)=p_2(u).
    \]
    Consequently, the initial-distribution condition in the definition
    of an \(\varepsilon\)-reduction gives
    \[
     \|P_{\epsilon}-Q_{\epsilon}\|_{\mathrm{TV}}
     =
     \|\phi_*p_1-p_2\|_{\mathrm{TV}}
     \le\varepsilon.
    \]
    
    Now fix \(\mathbf a\in \inputalph^n\), a next action \(b\in \inputalph\), and
    \((\mathbf o,u)\in  \outputalph^n\times \states_2\). Define a probability
    distribution \(\lambda_{\mathbf a,\mathbf o,u}\) on
    \(\phi^{-1}(u)\) as follows. If
    \(P_{\mathbf a}(\mathbf o,u)>0\), let
    \[
     \lambda_{\mathbf a,\mathbf o,u}(s)
     =
     \frac{
       \Pr{}_{T_1}\!
       \left(
          O_{1:n}=\mathbf o,\,
          S_n^1=s
          \mid \inputalph_{1:n}=\mathbf a
       \right)}
       {P_{\mathbf a}(\mathbf o,u)}.
    \]
    If \(P_{\mathbf a}(\mathbf o,u)=0\), set arbitrarily $\lambda_{\mathbf a,\mathbf o,u}
     =\delta_{r(u)}$. Then, $\lambda_{\textbf{a}, \textbf{o}, u}(s)$ represents the probability for the state of transducer $T_1$ to be $s$ at step $n$ conditioned on $T_1$ being at a state in $\phi^{-1}(u)$.
    
    For \(\mathbf o\in \outputalph^n\), let
    \[
     j_{\mathbf o}:\states_2\times \outputalph\longrightarrow \outputalph^{n+1}\times \states_2,
     \qquad
     j_{\mathbf o}(u',o)=(\mathbf o o,u'),
    \]
    where \(\mathbf o o\) denotes concatenation. Consider the Markov kernels
    \(\widehat K_{1,\mathbf a,b}\) and \(\widehat K_{2,b}\) from
    \(\outputalph^n\times \states_2\) to \(\outputalph^{n+1}\times \states_2\) given by
    \begin{align*}
    \widehat K_{1,\mathbf a,b}
     \bigl(\mathord{\cdot}\mid\mathbf o,u\bigr)
    &=
    (j_{\mathbf o})_*
    \left[
     \sum_{s\in\phi^{-1}(u)}
     \lambda_{\mathbf a,\mathbf o,u}(s)\,
     (\phi\times\operatorname{id}_\outputalph)_*
     \kappa_1(\mathord{\cdot},\mathord{\cdot}\mid s,b)
    \right],
    \\
    \widehat K_{2,b}
     \bigl(\mathord{\cdot}\mid\mathbf o,u\bigr)
    &=
    (j_{\mathbf o})_*
     \kappa_2(\mathord{\cdot},\mathord{\cdot}\mid u,b).
    \end{align*}
    The first kernel performs one step of \(T_1\), coarse-grains the next
    state through \(\phi\), and retains the already observed output
    prefix. The second kernel performs the corresponding operation for
    \(T_2\).
    
    For every \((\mathbf o,u)\), push-forward contraction, convexity of
    total variation, and the one-step condition of the
    \(\varepsilon\)-reduction give
    \begin{align}
    &\left\|
     \widehat K_{1,\mathbf a,b}
           (\mathord{\cdot}\mid\mathbf o,u)
     -
     \widehat K_{2,b}
           (\mathord{\cdot}\mid\mathbf o,u)
    \right\|_{\mathrm{TV}}
    \nonumber\\
    &\quad\le
    \left\|
     \sum_{s\in\phi^{-1}(u)}
     \lambda_{\mathbf a,\mathbf o,u}(s)
     (\phi\times\operatorname{id}_\outputalph)_*
           \kappa_1(\mathord{\cdot},\mathord{\cdot}\mid s,b)
     -
           \kappa_2(\mathord{\cdot},\mathord{\cdot}\mid u,b)
    \right\|_{\mathrm{TV}}
    \nonumber\\
    &\quad\le
    \sum_{s\in\phi^{-1}(u)}
     \lambda_{\mathbf a,\mathbf o,u}(s)
     \left\|
      (\phi\times\operatorname{id}_\outputalph)_*
           \kappa_1(\mathord{\cdot},\mathord{\cdot}\mid s,b)
      -
           \kappa_2(\mathord{\cdot},\mathord{\cdot}\mid\phi(s),b)
     \right\|_{\mathrm{TV}}
    \nonumber\\
    &\quad\le\varepsilon.
    \label{eq:repair-row-kernel-bound}
    \end{align}
    Moreover, by construction of the conditional distributions
    \(\lambda_{\mathbf a,\mathbf o,u}\) it follows that
    \[
     P_{\mathbf a b}
     =
     P_{\mathbf a}\widehat K_{1,\mathbf a,b},
     \qquad
     Q_{\mathbf a b}
     =
     Q_{\mathbf a}\widehat K_{2,b}.
    \]
    Thus, applying the kernel perturbation bound from Lemma~\ref{lem:tv_basic_properties} and then the
    induction hypothesis yields
    \begin{align*}
    \|P_{\mathbf a b}-Q_{\mathbf a b}\|_{\mathrm{TV}}
    &\le
     \|P_{\mathbf a}-Q_{\mathbf a}\|_{\mathrm{TV}}
     +
     \sup_{(\mathbf o,u)}
     \left\|
      \widehat K_{1,\mathbf a,b}(\mathord{\cdot}\mid\mathbf o,u)
      -
      \widehat K_{2,b}(\mathord{\cdot}\mid\mathbf o,u)
     \right\|_{\mathrm{TV}}
    \\
    &\le (n+1)\varepsilon+\varepsilon
     = (n+2)\varepsilon.
    \end{align*}
    This completes the induction.
    
    The output distribution
    \(D_{\interface_{T_1}}(\mathbf a)\) is the marginal of
    \(P_{\mathbf a}\) on \(\outputalph^n\), and
    \(D_{\interface_{T_2}}(\mathbf a)\) is the corresponding marginal of
    \(Q_{\mathbf a}\). Since marginalization contracts total variation,
    \[
     \left\|
       D_{\interface_{T_1}}(\mathbf a)
       -
       D_{\interface_{T_2}}(\mathbf a)
     \right\|_{\mathrm{TV}}
     \le (n+1)\varepsilon
    \]
    for every \(\mathbf a\in \inputalph^n\). Therefore
    \begin{align*}
    d_\gamma(\interface_{T_1},\interface_{T_2})
    &=
     \sum_{n=0}^{\infty}
     \gamma^n
     \sup_{\mathbf a\in \inputalph^n}
     \left\|
       D_{\interface_{T_1}}(\mathbf a)
       -
       D_{\interface_{T_2}}(\mathbf a)
     \right\|_{\mathrm{TV}}
    \\
    &\le
     \varepsilon
     \sum_{n=0}^{\infty}(n+1)\gamma^n
     =
     \frac{\varepsilon}{(1-\gamma)^2}.
    \end{align*}
\end{proof}

\begin{proof}[Proof of Proposition~\ref{prop:compose_approx_homo}]
    Let $\langle \phi_1,f_1,g_1\rangle$ be the $\varepsilon_1$-homomorphism from $T_1$ to $T_2$ and $\langle \phi_2,f_2,g_2\rangle$ the $\varepsilon_2$-homomorphism from $T_2$ to $T_3$. Define the homomorphism $\langle \phi=\phi_2\circ\phi_1,f=f_2\circ f_1,g=g_2\circ g_1\rangle$ from $T_1$ to $T_3$. We will prove that this is a $(\varepsilon_1 + \varepsilon_2)$-homomorphism.

    Fix \(s\in \states_1\) and \(a\in \inputalph_1\). By the composition rule for
    push-forwards, the triangle inequality, and contraction of total
    variation under push-forwards,
    \begin{align*}
    &\left\|
      (\phi\times g)_*
           \kappa_1(\mathord{\cdot},\mathord{\cdot}\mid s,a)
      -
           \kappa_3(\mathord{\cdot},\mathord{\cdot}
              \mid \phi(s),f(a))
     \right\|_{\mathrm{TV}}
    \\
    &\quad\le
    \left\|
     (\phi_2\times g_2)_*
     \left(
       (\phi_1\times g_1)_*
           \kappa_1(\mathord{\cdot},\mathord{\cdot}\mid s,a)
       -
           \kappa_2(\mathord{\cdot},\mathord{\cdot}
              \mid\phi_1(s),f_1(a))
     \right)
    \right\|_{\mathrm{TV}}
    \\
    &\qquad\quad+
    \left\|
     (\phi_2\times g_2)_*
           \kappa_2(\mathord{\cdot},\mathord{\cdot}
              \mid\phi_1(s),f_1(a))
     -
           \kappa_3(\mathord{\cdot},\mathord{\cdot}
              \mid\phi_2(\phi_1(s)),f_2(f_1(a)))
    \right\|_{\mathrm{TV}}
    \\
    &\quad\le \varepsilon_1+\varepsilon_2.
    \end{align*}
    The initial distributions satisfy
    \begin{align*}
    \|\phi_*p_1-p_3\|_{\mathrm{TV}}
    &\le
     \|(\phi_2)_*((\phi_1)_*p_1)-(\phi_2)_*p_2\|_{\mathrm{TV}}
     +
     \|(\phi_2)_*p_2-p_3\|_{\mathrm{TV}}
    \\
    &\le
     \|(\phi_1)_*p_1-p_2\|_{\mathrm{TV}}
     +
     \|(\phi_2)_*p_2-p_3\|_{\mathrm{TV}}
    \\
    &\le \varepsilon_1+\varepsilon_2.
    \end{align*}
    Thus
    \(\langle\phi_2\circ\phi_1,f_2\circ f_1,g_2\circ g_1\rangle\)
    is an \((\varepsilon_1+\varepsilon_2)\)-homomorphism.
    
    If both original maps are reductions, then the action and output maps
    are identities and \(\phi_1,\phi_2\) are surjective. Hence
    \(\phi_2\circ\phi_1\) is surjective, so the composition is an
    \((\varepsilon_1+\varepsilon_2)\)-reduction.
\end{proof}

\begin{proof}[Proof of Proposition~\ref{prop:ordinary-to-linear-approx}]
    For $i\in\{1,2\}$, let $V_i=\mathbb{R}^{S_i}$ with the $\ell_1$ norm, let
    $\xi_i=p_i$, and define
    \[
        \lambda_i(x)=\sum_{s\in S_i}x_s,
        \qquad
        M^{i}_{a,o}e_s=\sum_{t\in S_i}\kappa_i(t,o\mid s,a)e_t.
    \]
    Each $M^i_{a,o}$ is nonnegative and column-substochastic, and hence is an
    $\ell_1$ contraction. Moreover,
    $\|\xi_i\|_1=\|\lambda_i\|_{V_i^*}=1$, so these linearizations are contractive.
    
    Define $L_\phi e_s=e_{\phi(s)}$ and extend linearly. Since $\phi$ is
    surjective, so is $L_\phi$, and $\|L_\phi\|_{1\to1}=1$. The initial-state
    condition of the ordinary reduction gives
    \[
        \|L_\phi\xi_1-\xi_2\|_1
        =2\|\phi_*p_1-p_2\|_{\mathrm{TV}}
        \leq 2\varepsilon.
    \]
    For a fixed $s\in \states_1$ and $a \in \inputalph$, the vector
    \(
     (L_\phi M^1_{a,o}-M^2_{a,o}L_\phi)e_s
    \)
    is the $o$-component of the difference between the two joint laws on
    $\states_2\times \outputalph$ appearing in the definition of an ordinary
    $\varepsilon$-reduction. Namely, 
    \begin{align*}
        (L_\phi M^1_{a,o}-M^2_{a,o}L_\phi)e_s &= \sum_{t_1 \in \states_1} \kappa_1(t_1, o | s, a) e_{\phi(t_1)} - \sum_{t_2 \in \states_2} \kappa_2 (t_2, o| \phi(s), a) e_{t_2}\\
        &= \sum_{t_2 \in \states_2} \left[ \sum_{t_1 \in \phi^{-1}(t_2)} \kappa_1(t_1, o| s, a) - \kappa_2 (t_2, o | \phi(s), a)\right] e_{t_2}
    \end{align*}
    Consequently,
    \[
      \|(L_\phi M^1_{a,o}-M^2_{a,o}L_\phi)e_s\|_1
      \leq 2\varepsilon.
    \]
    Taking the maximum over the columns gives
    \[
      \|L_\phi M^1_{a,o}-M^2_{a,o}L_\phi\|_{1\to1}
      \leq 2\varepsilon.
    \]
    Finally, $\lambda_2L_\phi=\lambda_1$. Thus $L_\phi$ is a
    $2\varepsilon$-linear reduction. The factor $2$ shows up because of the normalization $\frac{1}{2}$ in total variation.
\end{proof}

\begin{proof}[Proof of Proposition~\ref{prop:linear-approx-composition}]
    Write
    \(
     G_i=(V_i,\xi_i,\lambda_i,\{M^i_{a,o}\}_{a,o})
    \).
    Since $L$ and $K$ are bounded and surjective, so is $KL$.
    For the initial vectors,
    \[
    \begin{split}
        \|KL\xi_0-\xi_2\|
        &\leq \|K\|\,\|L\xi_0-\xi_1\|+\|K\xi_1-\xi_2\|\\
        &\leq \|K\|\varepsilon_1+\varepsilon_2.
    \end{split}
    \]
    For each symbol $(a,o)$,
    \[
    \begin{split}
        KLM^0_{a,o}-M^2_{a,o}KL
        ={}&K(LM^0_{a,o}-M^1_{a,o}L)\\
           &+(KM^1_{a,o}-M^2_{a,o}K)L,
    \end{split}
    \]
    and hence the operator norm of this difference is at most
    \(
      \|K\|\varepsilon_1+\|L\|\varepsilon_2
    \).
    Finally,
    \[
        \lambda_2KL-\lambda_0
        =(\lambda_2K-\lambda_1)L+(\lambda_1L-\lambda_0),
    \]
    whose dual norm is at most $\|L\|\varepsilon_2+\varepsilon_1$. All three quantities are
    bounded by $c(K)\varepsilon_1+c(L)\varepsilon_2$.
\end{proof}

\begin{proof}[Proof of Theorem~\ref{thm:linear-discounted-continuity}]
    Let $L:V\to W$ be the $\varepsilon$-linear reduction. For
    $a=a_1\cdots a_n\in \inputalph^n$ and $o=o_1\cdots o_n\in \outputalph^n$, write $M_{a,o}=M_{a_n,o_n}\cdots M_{a_1,o_1}$ and $N_{a,o}=N_{a_n,o_n}\cdots N_{a_1,o_1}$. Symbol-wise contractivity implies $\|M_{a,o}\xi\|\leq1$. We claim that
    \begin{equation}
    \label{eq:linear-branch-state-error}
        \|LM_{a,o}\xi-N_{a,o}\xi'\|\leq(n+1)\varepsilon.
    \end{equation}
    For $n=0$, this is the initial-vector condition. If the claim holds at length
    $n$ and $\sigma=(a_{n+1},o_{n+1})$, then
    \[
    \begin{split}
     &\|LM_\sigma M_{a,o}\xi-N_\sigma N_{a,o}\xi'\|\\
     &\quad\leq
     \|(LM_\sigma-N_\sigma L)M_{a,o}\xi\|
     +\|N_\sigma(LM_{a,o}\xi-N_{a,o}\xi')\|\\
     &\quad\leq \varepsilon+(n+1)\varepsilon.
    \end{split}
    \]
    This proves \eqref{eq:linear-branch-state-error} by induction.
    
    For every output word $o\in \outputalph^n$,
    \[
    \begin{split}
     &|\lambda(M_{a,o}\xi)-\lambda'(N_{a,o}\xi')|\\
     &\quad\leq
     |(\lambda-\lambda'L)(M_{a,o}\xi)|
     +|\lambda'(LM_{a,o}\xi-N_{a,o}\xi')|\\
     &\quad\leq (n+2)\varepsilon.
    \end{split}
    \]
    Summing over the $|\outputalph|^n$ output words and dividing by two gives the second
    quantity in the minimum in \eqref{eq:linear-horizon-bound}, the bound by one
    holds because both sides are probability distributions.
    
    Equation~\eqref{eq:linear-general-modulus} follows by summing the finite horizon
    bounds. For each fixed $n$, the remaining summand tends to zero with $\varepsilon$ and is
    bounded by $\gamma^n$. Since $\sum_n\gamma^n<\infty$, we conclude that the right hand side converges to 0.
\end{proof}

\begin{proof}[Proof of Proposition~\ref{prop:no-exact-universal-target-below-one}]
    Fix $\inputalph=\{\ell,g\}$ and $\outputalph=\{\$,0,1\}$. Define $\interface$ as follows: in the first step, the output is the symbol $\$$ with probability 1. Then, a fair coin is thrown, and the output is always $0$ or $1$ for all the next steps, depending on this coin. Thus, for every
    $a_1\cdots a_k\in\inputalph^k$ with $k \geq 2$, we have
    \[
        \Pr{}_{\interface}(\$\mid a_1)=1
    \]
    and
    \[
        \Pr{}_{\interface}(\$0^{k-1}\mid a_1\cdots a_k)
        =
        \Pr{}_{\interface}(\$1^{k-1}\mid a_1\cdots a_k)
        =
        \frac12,
    \]
    Note that the interface is \textit{independent} of the actions taken.
    
    Consider the three-state transducer \(W\) with states $S_W=\{r,t_0,t_1\}$ and initial distribution centred at $r$ and with kernel
    \[
        \kappa_W(t_0,\$\mid r,a)=\frac12,
        \qquad
        \kappa_W(t_1,\$\mid r,a)=\frac12,
    \]
    \[
        \kappa_W(t_0,0\mid t_0,a)=1,
        \qquad
        \kappa_W(t_1,1\mid t_1,a)=1.
    \]
    for every $a \in \inputalph$. Clearly $W$ implements
    $\interface$, and it can be proven that there is no transducer with less than 3 states implementing this interface.
    
    Now, let's define another implementation $U$. Its states are $S_U=\{r,c_{00},c_{01},c_{10},c_{11},z_0,z_1\}$ with initial distribution centered at $r$. We describe the kernel by steps. First, we state that
    \[
        \kappa_U(c_{ij},\$\mid r,a)=\frac14
        \qquad
        (i,j\in\{0,1\}).
    \]
    for every $a \in \inputalph$. Namely, in the first step the transducer transitions with uniform probability to any of the states $c_{ij}$.
    
    From $c_{ij}$, the action $\ell$ reads the first coordinate, while action $g$ reads
    the second one. More precisely, we have
    \[
        \kappa_U(z_i,i\mid c_{ij},\ell)=1,
        \qquad
        \kappa_U(z_j,j\mid c_{ij},g)=1.
    \]
    
    Finally, for every $a\in\inputalph$ we set
    \[
        \kappa_U(z_0,0\mid z_0,a)=1,
        \qquad
        \kappa_U(z_1,1\mid z_1,a)=1.
    \]
    
    It can be checked that the transducer $U$ also implements
    $\interface$: after the first output $\$$, the pair $(i,j)$ is uniformly chosen; and whichever coordinate is read by the second action the final result is a fair bit. Afterwards, the machine moves to $z_0$ or $z_1$, where the same bit is repeated forever.
    
    Now suppose, towards a contradiction, that some $T\in\mathcal L_{\interface}^{0,d}$ receives
    a $\delta$-reduction from every element of $\mathcal L_{\interface}^{0,d}$, with
    $\delta<1$. Since $W\in\mathcal L_{\interface}^{0,d}$, there is a surjective state map
    $S_W\to S_T$. Hence $|S_T|\le |S_W|=3$,
    and the fact that any implementation of $\interface$ must have at least three states implies that $|S_T| = 3$. 
    
    The three states of $T$ can
    be labelled $x_{\$},x_0,x_1$ depending on which node from $W$ is the one mapped to them through the $\delta$-reduction. Note that it must be the case that 
    \[
        \Pr{}_{T}(\star|x_{\star}, a) = 1
    \]
    for every $\star \in \outputalph$. To see this, first note that there must be some state which outputs $\$$ with probability one. Otherwise, it would be impossible for $T$ to implement $\interface$ exactly. With the same reasoning we can see that there must be some state that always outputs $0$ and another one that always outputs $1$. Then, we conclude that there is only one possibility for the reduction from $W$ to $T$ considering that $\delta < 1$.
    
    Since $U\in\mathcal L_{\interface}^{0,d}$, there is a
    $\delta$-reduction $\psi:U\reducesTo{\delta}T$. 
    Consider the state $c_{01}\in S_U$. There are three possible images, and we go through them one by one. 
    
    If $\psi(c_{01})=x_{\$}$ we reach an absurd, since the distributions between those states are at distance $1$: $c_{01}$ assigns 0 probability to outputting $\$$. If $\psi(c_{01})=x_0$, then, under action $g$, the state $c_{01}$ outputs $1$ with
    probability one, while $x_0$ outputs $0$ with probability one. Thus $\psi$ is not a proper $\delta$-reduction with $\delta < 1$. If \(\psi(c_{01})=x_1\), we can argue in the same way.
    
    Therefore, no such transducer \(T\in\mathcal L_{\interface}^{0,d}\) exists.
\end{proof}

\begin{proof}[Proof of Proposition~\ref{prop:no-ordinary-local-minimization}]
    Let \(W\) and \(U\) be the two exact implementations of \(\interface\) constructed in the proof of
    Proposition~\ref{prop:no-exact-universal-target-below-one}. Since $\interface_W=\interface_U=\interface$,
    we have $W,U\in\mathcal L_{\interface}^{\varepsilon, d_{\infty}}$
    for every $\varepsilon\ge 0$. Thus it is enough to prove the following claim: if $W$ $\delta$-reduces to $C$ and $U$ $\delta$-reduces to $C$, then $\delta \geq \frac{1}{2}$.

    Let $\varphi:S_W\to S_C$
    be the state map of a $\delta$-reduction from $W$ to $C$. Since
    ordinary reductions are surjective on states, we have $|S_C|\le |S_W|=3$.
    
    First suppose that \(|S_C|\le 2\). The three states \(r,t_0,t_1\) of \(W\) have one-step
    output marginals $\delta_{\$}$, $\delta_0$ and $\delta_1$ respectively, under every action. Since there are at most two states in \(C\), two of
    \(r,t_0,t_1\) must have the same image \(x\in S_C\). Hence, for two distinct outputs
    \(o\neq o'\), the output marginal \(\operatorname{out}_C(x,a)\) is within total variation
    distance \(\delta\) of both \(\delta_o\) and \(\delta_{o'}\). Marginalization cannot
    increase total variation, so
    \[
        1
        =
        \|\delta_o-\delta_{o'}\|_{\mathrm{TV}}
        \le
        \|\delta_o-\operatorname{out}_C(x,a)\|_{\mathrm{TV}}
        +
        \|\operatorname{out}_C(x,a)-\delta_{o'}\|_{\mathrm{TV}}
        \le
        2\delta.
    \]
    Thus \(\delta\ge 1/2\).
    
    It remains to consider the case $|S_C|=3$, where $\varphi$ is bijective. Write
    \[
        x_{\$}=\varphi(r),
        \qquad
        x_0=\varphi(t_0),
        \qquad
        x_1=\varphi(t_1).
    \]
    as before. The reduction $W\reducesTo{\delta}C$ implies that, for every action $a$,
    \[
        \|\operatorname{out}_C(x_{\$},a)-\delta_{\$}\|_{\mathrm{TV}}\le\delta,
    \]
    \[
        \|\operatorname{out}_C(x_0,a)-\delta_0\|_{\mathrm{TV}}\le\delta,
        \qquad
        \|\operatorname{out}_C(x_1,a)-\delta_1\|_{\mathrm{TV}}\le\delta.
    \]
    
    Let $\psi:S_U\to S_C$ be the map of a $\delta$-reduction from $U$ to $C$. Consider the
    state \(c_{01}\in S_U\). There are three possibilities.
    
    If \(\psi(c_{01})=x_{\$}\), then under action \(\ell\), the state \(c_{01}\) outputs
    \(0\) with probability one. The reduction \(U\reducesTo{\delta}C\) gives
    \[
        \|\delta_0-\operatorname{out}_C(x_{\$},\ell)\|_{\mathrm{TV}}\le\delta.
    \]
    Together with the estimate coming from $\varphi$, we have
    \[
        \|\operatorname{out}_C(x_{\$},\ell)-\delta_{\$}\|_{\mathrm{TV}}\le\delta,
    \]
    and then
    \[
        1
        =
        \|\delta_0-\delta_{\$}\|_{\mathrm{TV}}
        \le
        2\delta.
    \]
    
    If \(\psi(c_{01})=x_0\), then under action \(g\), the state \(c_{01}\) outputs \(1\)
    with probability one. Hence
    \[
        \|\delta_1-\operatorname{out}_C(x_0,g)\|_{\mathrm{TV}}\le\delta.
    \]
    But \(x_0\) is \(\delta\)-close to a \(0\)-state, so
    \[
        \|\operatorname{out}_C(x_0,g)-\delta_0\|_{\mathrm{TV}}\le\delta.
    \]
    Therefore
    \[
        1
        =
        \|\delta_1-\delta_0\|_{\mathrm{TV}}
        \le
        2\delta.
    \]
    
    The last case can be treated in the same way.
\end{proof}

\begin{proof}[Proof of Lemma~\ref{lem:atomic-contractivity}]
    For every \(w,v\in\Sigma^*\) it is the case that $0\le h_\interface(w)(v)=F_\interface(wv)\le1$. Hence, if $r=\sum_{i=1}^m\alpha_i h_\interface(w_i)$, then
    \begin{align*}
    \|r\|_{\mathrm{pred}}
    &=
    \sup_{v\in\Sigma^*}
    \left|
    \sum_{i=1}^m\alpha_i h_\interface(w_i)(v)
    \right|
    \\
    &\le
    \sum_{i=1}^m|\alpha_i|
      \sup_{v\in\Sigma^*}|F_\interface(w_iv)|
    \\
    &\le
    \sum_{i=1}^m|\alpha_i|.
    \end{align*}
    Taking the infimum over all atomic decompositions of \(r\) gives $\|r\|_{\mathrm{pred}}
    \le
    \|r\|_{\mathrm{at},\interface}$.
    
    It is straightforward to prove that $\|\lambda r\|_{\mathrm{at}, \interface} = |\lambda| \|r\|_{\mathrm{at}, \interface}$, and the triangle inequality
    is also easy to prove by concatenating atomic decompositions of the two summands.
    Finally, if \(\|r\|_{\mathrm{at},\interface}=0\), the preceding inequality
    implies that \(\|r\|_{\mathrm{pred}}=0\). Thus \(r(v)=0\) for every
    \(v\in\Sigma^*\), and hence \(r=0\). Therefore
    \(\|\mathord{\cdot}\|_{\mathrm{at},\interface}\) is a norm.
    
   We now check that $\|\xi_\interface\|_{\mathrm{at},\interface}=1$. Since \(\xi_\interface=h_\interface(\epsilon)\) is itself an atom, $\|\xi_\interface\|_{\mathrm{at},\interface}\le1$. On the other hand, we have $\|\xi_\interface\|_{\mathrm{pred}}
    \ge
    |\xi_\interface(\epsilon)|
    =
    F_\interface(\epsilon)
    =
    1$. Thus, using the previous shown relation between the atomic and predictive norms we conclude that $\|\xi_\interface\|_{\mathrm{at},\interface}=1$.
    
    For the readout functional, recall that
    \(\lambda_\interface(r)=r(\epsilon)\). Thus
    \[
    |\lambda_\interface(r)|
    =
    |r(\epsilon)|
    \le
    \|r\|_{\mathrm{pred}}
    \le
    \|r\|_{\mathrm{at},\interface},
    \]
    and consequently $\|\lambda_\interface\|_{(V_\interface,\|\mathord{\cdot}\|_{\mathrm{at},\interface})^*}\le1$.
    Since $\lambda_\interface(\xi_\interface)=F_\interface(\epsilon)=1$ and $\|\xi_\interface\|_{\mathrm{at},\interface}=1$ the reverse inequality also holds, and we obtain $\|\lambda_\interface\|_{(V_\interface,\|\mathord{\cdot}\|_{\mathrm{at},\interface})^*}=1$.
    
    Finally, we prove that the shift operators are nonexpansive in the atomic norm. For any atomic decomposition
    \(r=\sum_{i=1}^m\alpha_i h_\interface(w_i)\), the definition of the shift gives
    \[
    R_\sigma^\interface r
    =
    \sum_{i=1}^m\alpha_i h_\interface(w_i\sigma).
    \]
    Every \(h_\interface(w_i\sigma)\) is again a Hankel-row atom, and hence
    \[
    \|R_\sigma^\interface r\|_{\mathrm{at},\interface}
    \le
    \sum_{i=1}^m|\alpha_i|.
    \]
    Taking the infimum over all atomic decompositions of \(r\) proves $\|R_\sigma^\interface r\|_{\mathrm{at},\interface}
    \le
    \|r\|_{\mathrm{at},\interface}$.
    
    Using all the results, we can conclude that $G_{\interface}^{\mathrm{at}}$ is a contractive transducer.
\end{proof}

\begin{proof}[Proof of Lemma~\ref{lem:atomic-hankel-row-approximation}]
    Because the vectors $h_{\mathcal I}(p_1),\ldots,h_{\mathcal I}(p_d)$ form a basis of $V_{\mathcal I}$, every Hankel row of
    $\mathcal I$ is reconstructed from its values on the selected suffixes. Namely,
    \[
        h_{\mathcal I}(w)=\operatorname{ev}_{\interface}(h_{\mathcal I}(w))C_{\mathcal I}^{-1}B_{\mathcal I}.
    \]
    Consequently,
    \[
        \Pi_{\mathcal J\to\mathcal I}h_{\mathcal J}(w)-h_{\mathcal I}(w)=e_wC_{\mathcal I}^{-1}B_{\mathcal I}
    \]
    where $e_w=\operatorname{ev}_\interfaceJ(h_{\mathcal J}(w))-\operatorname{ev}_{\interface}(h_{\mathcal I}(w))$. The condition $d_\infty(\mathcal I,\mathcal J)\leq\varepsilon$ implies
    $\|e_w\|_\infty\leq\varepsilon$. Since the rows of $B_{\mathcal I}$ are atoms for
    $\|\cdot\|_{\mathrm{at},\mathcal I}$,
    \[
    \begin{split}
        \|e_wC_{\mathcal I}^{-1}B_{\mathcal I}\|_{\mathrm{at},\mathcal I}
        &\leq\|e_wC_{\mathcal I}^{-1}\|_1\\
        &\leq\Gamma_{\mathcal I}\varepsilon.
    \end{split}
    \]
    
    The images of $h_{\mathcal J}(p_1),\ldots,h_{\mathcal J}(p_d)$ have coordinate matrix
    $C_{\mathcal J}C_{\mathcal I}^{-1}$ in the basis given by the rows of $B_{\mathcal I}$. If $C_{\mathcal J}$ is
    invertible, these images span $V_{\mathcal I}$, proving surjectivity.
\end{proof}

\begin{proof}[Proof of Theorem~\ref{thm:finite-rank-hankel-stability}]
    Choose $\overline\varepsilon(\mathcal I)>0$ so that
    $d\overline\varepsilon(\mathcal I)<\sigma_{\min}(C_{\mathcal I})$, where $\sigma_{\min}(C_\interface)$ denotes the minimal singular value of $C_{\interface}$. Then, using standard perturbation arguments we may conclude that $C_{\mathcal J}$ is invertible, and
    Lemma~\ref{lem:atomic-hankel-row-approximation} shows that
    $\Pi_{\mathcal J\to\mathcal I}$ is surjective. More precisely, note that, because $d_{\infty}(\interface, \interfaceJ) \leq \varepsilon$,
    \[
    \|C_\interfaceJ-C_\interface\|_2
    \leq
    \|C_\interfaceJ-C_\interface\|_F
    \leq
    d\varepsilon
    <
    \sigma_{\min}(C_\interface),
    \]
    and therefore, for every $x \neq 0$,
    \begin{align*}
        \|C_{\interfaceJ} x\|_2 &\geq \|C_\interface x\|_2 - \|(C_\interfaceJ - C_\interface)x\|_2\\
        &\geq \sigma_{\min}(C_\interface)\|x\|_2 - \|C_\interfaceJ - C_\interface\|_2 \|x\|_2\\
        &\geq (\sigma_{\min}(C_\interface) - d\varepsilon)\|x\|_2 > 0.
    \end{align*}
    
    Applying Lemma~\ref{lem:atomic-hankel-row-approximation} to the empty word gives
    \[
        \|\Pi_{\mathcal J\to\mathcal I}\xi_{\mathcal J}-\xi_{\mathcal I}\|_{\mathrm{at},\mathcal I}
        \leq\Gamma_{\mathcal I}\varepsilon.
    \]
    For $\sigma\in\Sigma$ and $w\in\Sigma^*$, let
    \(
     E_w=\Pi_{\mathcal J\to\mathcal I}h_{\mathcal J}(w)-h_{\mathcal I}(w)
    \). Since
    \(
     R^{\mathcal I}_\sigma h_{\mathcal I}(w)=h_{\mathcal I}(w\sigma)
    \), we have
    \[
    \begin{split}
     &(\Pi_{\mathcal J\to\mathcal I}R^{\mathcal J}_\sigma-R^{\mathcal I}_\sigma\Pi_{\mathcal J\to\mathcal I})h_{\mathcal J}(w)\\
     &\qquad=E_{w\sigma}-R^{\mathcal I}_\sigma E_w.
    \end{split}
    \]
    By Lemma~\ref{lem:atomic-contractivity}, the shifts of $G_{\mathcal I}$ are non-expansive with respect to the atomic norm, so
    \begin{equation}
    \label{eq:hankel-atomic-dynamics-on-atoms}
        \|(\Pi_{\mathcal J\to\mathcal I}R^{\mathcal J}_\sigma-R^{\mathcal I}_\sigma\Pi_{\mathcal J\to\mathcal I})h_{\mathcal J}(w)\|_{\mathrm{at},\mathcal I}
        \leq2\Gamma_{\mathcal I}\varepsilon.
    \end{equation}
    If $r=\sum_i\alpha_i h_{\mathcal J}(w_i)$, linearity and
    \eqref{eq:hankel-atomic-dynamics-on-atoms} give
    \[
    \begin{split}
     &\|(\Pi_{\mathcal J\to\mathcal I}R^{\mathcal J}_\sigma-R^{\mathcal I}_\sigma\Pi_{\mathcal J\to\mathcal I})r\|_{\mathrm{at},\mathcal I}\\
     &\qquad\leq2\Gamma_{\mathcal I}\varepsilon\sum_i|\alpha_i|.
    \end{split}
    \]
    Taking the infimum over all atomic decompositions of $r$ proves the required
    bound.
    
    The choice $q_1=\epsilon$ makes the readout exact. Indeed, the first column of
    $C_{\mathcal I}$ is $B_{\mathcal I}(\epsilon)$, and hence
    \[
    \begin{split}
        \lambda_{\mathcal I}(\Pi_{\mathcal J\to\mathcal I}r)
        &=(\Pi_{\mathcal J\to\mathcal I}r)(\epsilon)\\
        &=\operatorname{ev}_{\interfaceJ}(r)C_{\mathcal I}^{-1}B_{\mathcal I}(\epsilon)\\
        &=\operatorname{ev}_{\interfaceJ}(r)e_1=r(\epsilon)=\lambda_{\mathcal J}(r).
    \end{split}
    \]
    Thus all three defects are bounded by $2\Gamma_{\mathcal I}\varepsilon$.

    Finally, we show that $\Pi_{\interfaceJ \to \interface}$ is bounded. For every \(w\in\Sigma^*\),
    \[
    \|\Pi_{\interfaceJ\to \interface}h_\interfaceJ(w)\|_{\mathrm{at},\interface}
    \leq
    \|h_\interface(w)\|_{\mathrm{at},\interface}
    +
    \|\Pi_{\interfaceJ\to \interface}h_\interfaceJ(w)-h_\interface(w)\|_{\mathrm{at},\interface}
    \leq
    1+\Gamma_\interface\varepsilon.
    \]
    Consequently, if
    \(r=\sum_{i=1}^m\alpha_i h_\interfaceJ(w_i)\), then
    \[
    \|\Pi_{\interfaceJ\to \interface}r\|_{\mathrm{at},\interface}
    \leq
    (1+\Gamma_\interface\varepsilon)
    \sum_{i=1}^m|\alpha_i|.
    \]
    Taking the infimum over all atomic decompositions of \(r\) we conclude that $\|\Pi_{J\to I}\|\leq1+\Gamma_I\varepsilon$.
\end{proof}

\begin{proof}[Proof of Proposition~\ref{prop:predictive-rare-history-obstruction}]
    Let $\outputalph=\{0,1\}$, $\inputalph = \{\star\}$ (i.e. the dynamics are actionless) and let $\interface$ be the process of independent fair bits.
    Write $D_{\mathcal K}^{(m)}$ for the length-$m$ output law of an interface $\mathcal K$. For $n\geq 1$, define $\interfaceJ_n$ as follows. Its first $n$ outputs are independent
    fair bits. If these outputs are $0^n$, then every later output is $0$; otherwise, all later
    outputs continue to be independent fair bits.
    
    The length-$m$ distributions agree for $m\leq n$. For $m>n$ they differ only on strings
    beginning with $0^n$, and a direct calculation gives
    \[
        \left\|D_{\interface}^{(m)}-D_{\interfaceJ_n}^{(m)}\right\|_{\mathrm{TV}}
        =2^{-n}-2^{-m}.
    \]
    Taking the supremum over $m$ gives
    $d_\infty(\interface,\interfaceJ_n)=2^{-n}$.
    
    The transducer $\mathsf E(\interface)$ has one state, whose output law is
    $\frac12\delta_0+\frac12\delta_1$. In $\mathsf E(\interfaceJ_n)$, the state reached after
    $0^n$ outputs $0$ deterministically. Any state map to the one-state target must send this
    state to the unique state of $\mathsf E(\interface)$. Marginalizing the joint one-step
    kernels to outputs therefore gives
    \[
        \delta
        \geq
        \left\|\delta_0-\left(\tfrac12\delta_0+\tfrac12\delta_1\right)\right\|_{\mathrm{TV}}
        =\frac12.
    \]
    
    For the last claim, surjectivity of a reduction from the one-state transducer
    $\mathsf E(\interface)$ forces $C$ to have one state. Let $\nu$ be its output law. The two
    reductions imply
    \[
        \left\|\nu-\left(\tfrac12\delta_0+\tfrac12\delta_1\right)\right\|_{\mathrm{TV}}
        \leq\delta,
        \qquad
        \|\nu-\delta_0\|_{\mathrm{TV}}\leq\delta.
    \]
    Then, the triangle inequality gives $1/2\leq 2\delta$.
\end{proof}

\begin{proof}[Proof of Theorem~\ref{thm:predictive-residual-stability}]
    Since \(d_{\mathrm{res}}(\interface,\interfaceJ)\leq\varepsilon<1\), the definition of
    \(d_{\mathrm{res}}\) implies that $\operatorname{supp}(\interface)=\operatorname{supp}(\interfaceJ)$. Moreover, for every \(h\) in this common support we have $d_\infty(\interface^h,\interfaceJ^h)\leq \varepsilon$.
    
    We first prove that \(\phi\) is well defined. Suppose that
    \([h]_\interfaceJ=[u]_\interfaceJ\). By definition of predictive equivalence,
    \(\interfaceJ^h=\interfaceJ^u\). Therefore, by the triangle inequality,
    \begin{align*}
    d_\infty(\interface^h,\interface^u)
    &\leq
    d_\infty(\interface^h,\interfaceJ^h)
    +
    d_\infty(\interfaceJ^h,\interfaceJ^u)
    +
    d_\infty(\interfaceJ^u,\interface^u)
    \\
    &\leq 2\varepsilon
    <\Delta_\interface.
    \end{align*}
    If \([h]_\interface\neq [u]_\interface\), the definition of \(\Delta_\interface\) would instead
    give \(d_\infty(\interface^h,\interface^u)\geq\Delta_\interface\), which is a contradiction. Hence
    \([h]_\interface=[u]_\interface\), proving that \(\phi\) is well defined.
    
    The map is surjective. Indeed, every state of \(E(\interface)\) is of the
    form \([h]_\interface\) for some \(h\in\operatorname{supp}(\interface)\). Since the
    interfaces have the same support,
    \(h\in\operatorname{supp}(\interfaceJ)\), and therefore
    \[
    [h]_\interface=\phi([h]_\interfaceJ).
    \]
    
    It remains to verify the approximate one-step condition. For an
    interface \(K\), write
    \[
    \mu_K^h(o\mid a)=\Pr{}_{K^h}(o\mid a)
    \]
    for the one-step output law after \(h\). Fix
    \(h\in\operatorname{supp}(\interface)=\operatorname{supp}(\interfaceJ)\) and \(a\in \inputalph\).
    For every \(o\) having positive conditional probability, the
    canonical transducers move respectively to $[h(a,o)]_\interfaceJ$ and
    $[h(a,o)]_\interface$. By definition of \(\phi\), it holds that $\phi([h(a,o)]_\interfaceJ)=[h(a,o)]_\interface$.
    Furthermore, equality of supports implies that
    \(\mu_\interface^h(o\mid a)>0\) if and only if
    \(\mu_\interfaceJ^h(o\mid a)>0\). Consequently, after pushing the kernel of
    \(E(\interfaceJ)\) forward through \(\phi\), both joint kernels place their
    mass corresponding to \(o\) on the same pair $\bigl([h(a,o)]_\interface,o\bigr)$.
    It follows that
    \begin{align*}
    &\left\|
    (\phi\times\operatorname{id}_\outputalph)_*
    \kappa_{\epsilon,\interfaceJ}
      (\mathord{\cdot},\mathord{\cdot}\mid[h]_\interfaceJ,a)
    -
    \kappa_{\epsilon,\interface}
      (\mathord{\cdot},\mathord{\cdot}\mid[h]_\interface,a)
    \right\|_{\mathrm{TV}}
    \\
    &\hspace{4em}
    =
    \left\|
    \mu_\interfaceJ^h(\mathord{\cdot}\mid a)
    -
    \mu_\interface^h(\mathord{\cdot}\mid a)
    \right\|_{\mathrm{TV}}
    \\
    &\hspace{4em}
    \leq
    d_\infty(\interfaceJ^h,\interface^h)
    \leq\varepsilon.
    \end{align*}
    The initial state is preserved exactly:
    \[
    \phi_*\delta_{[\epsilon]_\interfaceJ}
    =
    \delta_{[\epsilon]_\interface}.
    \]
    Thus \(\phi\) determines an \(\varepsilon\)-reduction
    \(E(\interfaceJ)\to E(\interface)\).
\end{proof}

\begin{proof}[Proof of Corollary~\ref{coro:existence_delta_min_predictive}]
    Let $T$ be a predictive transducer from the statement. By Proposition~\ref{prop:epsilon_transducer_minimal} there is an exact reduction from $T$ to $\mathsf E(\interfaceJ)$. Then, by Theorem~\ref{thm:predictive-residual-stability} there is a $\varepsilon$-reduction from $E(\interfaceJ)$ to $E(\interface)$. Composing them we get the desired result, using Proposition~\ref{prop:compose_approx_homo} to bound the error of the composition.
\end{proof}

\bibliographystyle{plain}
\bibliography{biblio}

\end{document}